\documentclass{article}

\PassOptionsToPackage{numbers, compress}{natbib}

\usepackage[preprint]{neurips_2026}

\usepackage[utf8]{inputenc} 
\usepackage[T1]{fontenc}    
\usepackage{hyperref}       
\usepackage{url}            
\usepackage{booktabs}       
\usepackage{amsfonts}       
\usepackage{nicefrac}       
\usepackage{microtype}      
\usepackage{xcolor}         
\usepackage{colortbl}
\usepackage{multirow}
\usepackage{graphicx}
\usepackage{amsmath}
\usepackage{amssymb}
\usepackage{mathtools}
\usepackage{amsthm}
\usepackage{enumitem}
\usepackage{array}
\usepackage{textcomp}
\usepackage{amssymb} 
\usepackage{caption}
\usepackage{wrapfig}
\usepackage{needspace}
\usepackage{graphicx}
\newtheorem{theorem}{Theorem}[section]

\newtheorem{lemma}{Lemma}[section]

\title{REACT: Physically and Chemically Consistent Reconstruction of Marine Active Tracers}

\author{%
  \mdseries
  Wenbin Dai\thanks{Equal contribution.},
  Hao Zheng\footnotemark[1],
  Shiyu Liang\footnotemark[1]\hspace{0.25em}\thanks{Corresponding author.},
  Chaofan Sun,
  Xueying Zhang,\\
  Hanbo Huang,
  Xuan Gong,
  Yiran Zhang,
  Enhui Liao\\
  Shanghai Jiao Tong University\\
  Shanghai, China
}

\begin{document}

\maketitle

\begin{abstract}

Reconstructing global sea surface pH from sparse observations is critical for monitoring ocean acidification and understanding marine carbon cycling. Traditional assimilation and inverse models are physically grounded but costly for large-scale reconstruction. Recent black-box and physics-guided AI models improve efficiency, but are mainly designed for passive tracers, where the reconstructed variable is also the transported inventory. In contrast, pH is an active carbonate tracer: it is the prediction target, while dissolved inorganic carbon (DIC) is the conserved carbon inventory. This mismatch can produce low pH error while violating carbonate closure and source-free carbon conservation. To address this, we introduce \textbf{REACT}, a carbon-first reconstruction framework that decouples transport, active correction, and chemical decoding. REACT transports a latent carbonate state with a conservative advection--diffusion solver, captures non-conservative carbon-cycle variations with a source module, decodes the corrected state into pH, and constrains the output through carbonate equilibrium. This design keeps pH as the target while enforcing consistency on the underlying carbon state. On simulation data, REACT reduces pH NRMSE by \(14.7\%\) and chemical consistency error by \(24.0\%\) over the best baseline. Cross-temporal-scale evaluations show robustness against error accumulation from coarse to fine temporal scales, and ablation studies validate the effectiveness of each component.

\end{abstract}

\section{Introduction}
\label{sec:introduction}
                                                                          
Reconstructing global sea surface pH from sparse observations is essential for monitoring ocean acidification and marine carbon cycling. However, direct pH measurements remain highly sparse and irregular, with valid records covering only \(0.11\%\) of the global spatiotemporal grid according to the World Ocean Database~(WOD)\citep{WOD}. This challenge is further complicated by the fact that pH is not an unconstrained scalar variable, but a chemical coordinate of the ocean carbonate system governed by transport, biological activity, and carbonate equilibrium~\citep{zeebe2001co2,dickson2010carbon,humphreys2022pyco2sys}. Consequently, reconstruction models must infer dense pH fields from limited observations while preserving consistency with the physical and chemical structure of the carbonate system.

Recent advances in deep learning have improved the reconstruction of geophysical fields. CNNs and transformers capture spatial or spatiotemporal correlations from gridded data~\citep{ronneberger2015unet,vit,gao2022earthformer}. Inpainting models and climate foundation models provide powerful generic completion tools~\citep{zhao2021large-CoModGAN,repaint,prithvi,atmorep}. Physics-guided methods further introduce PDE constraints, neural operators, or variational consistency~\citep{fno,marble,luo2024inr,beauchamp20234dvarnet,nuts}. However, most of these methods are designed for passive tracers, where the reconstructed variable is also the transported inventory. This assumption does not hold for pH.

Reconstructing sea surface pH requires a fundamentally different approach. Unlike conserved tracers, pH is a chemical coordinate of the carbonate system. The conserved carbon inventory is closer to dissolved inorganic carbon (DIC), while pH is determined by carbonate equilibrium jointly with variables such as surface \(p\mathrm{CO}_2\). This induces a structural mismatch: direct transport of pH may produce smooth fields with negligible error, while allowing the implied DIC to drift and potentially violating carbonate closure. Therefore, pH error alone is insufficient. A valid reconstruction should simultaneously recover accurate pH fields, maintain consistency with carbonate chemistry and preserve conservation in the source-free carbon state.

The main challenge is to disentangle the three mechanisms jointly represented by the pH observations. Conservative transport redistributes carbon through ocean flow, while non-conservative processes, including air--sea exchange, biological activity, vertical mixing, and unresolved reactions, modify the carbonate state. Carbonate equilibrium then maps the carbon state to pH. Modeling pH as a passive tracer therefore applies transport dynamics to an inappropriate variable, whereas fully black-box reconstruction avoids this mismatch at the cost of providing limited interpretability. A principled formulation should instead transport a carbon-related state, represent non-conservative effects through a source module, and decode pH subject to carbonate-system constraints.

We propose \textbf{REACT}, short for \textbf{RE}constructing \textbf{A}ctive \textbf{C}arbonate \textbf{T}racers, to address this mismatch via a carbon-first reconstruction principle. At its core is a latent carbonate state that serves as the transported physical coordinate, while pH remains the decoded target. The latent state is propagated through a conservative advection--diffusion solver, corrected by an active source module for non-conservative carbon-cycle effects, and decoded to pH under the carbonate equilibrium constraints of DIC and \(p\mathrm{CO}_2\). This design separates transport, active correction, and chemical decoding, preserving source-free carbon mass without imposing conservation on pH. We further show that, under standard conditions, small chemical and physical losses certify carbonate closure and represented DIC conservation, a guarantee that passive pH transport generally cannot provide. Empirically, REACT consistently improves pH reconstruction and chemical consistency on both simulation-based cross-temporal tasks and real-world sparse observations, while ablations validate the role of source correction, mass conservation, and chemical closure.
Our contributions are as follows:
\begin{itemize}[leftmargin=*]
    \item We formulate sea surface pH mapping as \textit{active carbonate-tracer reconstruction}, revealing the mismatch between the reconstructed variable, pH, and the conserved carbon inventory, DIC.

    \item We propose \textbf{REACT}, a carbon-first framework that combines latent carbonate-state transport, active source correction, pH decoding, and carbonate-system closure. We provide a theoretical analysis showing why carbon-state transport can satisfy chemical closure and source-free DIC conservation.

    \item We empirically demonstrate that REACT achieves state-of-the-art performance on both simulation-based and real-world sparse observation reconstruction. On simulation data, REACT reduces pH NRMSE by \(14.7\%\) and chemical consistency error by \(24.0\%\) over the best baseline.

    \item Cross-temporal-scale evaluations and ablation studies further validate the design of REACT, showing robustness against coarse-to-fine error accumulation and confirming the effectiveness of source correction, mass conservation and chemical closure.
\end{itemize}

\section{Related Work}

\textbf{Assimilation and Inversion Models.}
Traditional data assimilation and inverse modeling methods, including optimal interpolation~\citep{conkright2002world-OI}, 3D/4D-Var~\citep{courtier1994strategy-4D-var}, ensemble Kalman filtering~\citep{nerger2008improving-KalmanFilter}, and variational inverse models~\citep{brasseur1991application-VIM}, reconstruct continuous ocean states by combining sparse observations with background dynamics. They provide strong physical consistency, {but depend on iterative optimization, prescribed error assumptions, and repeated numerical model runs, making long-term high-resolution reconstruction \textit{computationally costly}}.

\textbf{Black-box Models.}
Data-driven reconstruction methods replace expensive model integration with learned mappings. CNN-based models such as U-Net~\citep{ronneberger2015unet} and transformer-based models such as ViT~\citep{vit} and Earthformer~\citep{gao2022earthformer} capture spatial or spatiotemporal correlations from gridded fields. General-purpose inpainting methods based on GANs~\citep{zhao2021large-CoModGAN} and diffusion models~\citep{repaint} improve generic completion quality, while foundation models such as Prithvi~\citep{prithvi} and AtmoRep~\citep{atmorep} learn representations from large-scale weather and climate data. However, these models learn physical structure only implicitly from data, \textit{without enforcing transport, conservation, or carbonate closure}. They can fill missing fields, but they do not guarantee physically faithful ocean reconstruction.

\textbf{Physics-constrained Models.}
Physics-constrained models further introduce physical priors into learning. Neural operators such as FNO~\citep{fno}, neural field models such as GridMix~\citep{marble}, and implicit neural representations~\citep{luo2024inr} model PDE-governed fields, while 4DVarNet~\citep{beauchamp20234dvarnet} introduces variational consistency and NUTS~\citep{nuts} explicitly models transport dynamics for ocean reconstruction. However, these methods place dynamics on the reconstructed variable itself. This works for passive tracers, but \textit{fails for active carbonate tracers}: pH is the target, while carbon inventory is the physical state.

\begin{figure}
    \centering
    \captionsetup{font=footnotesize,skip=2pt}
    \includegraphics[width=\linewidth]{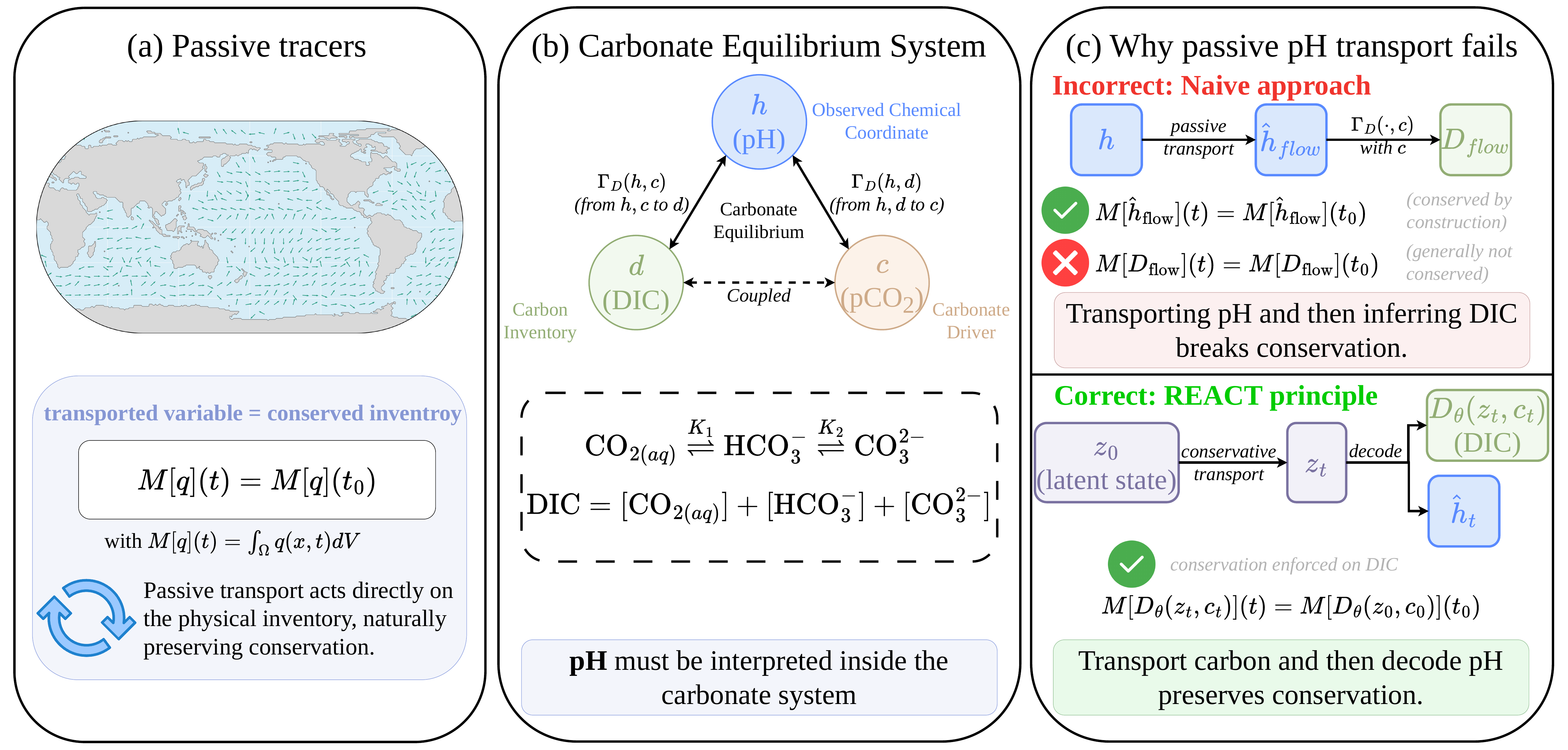}
    \caption{\textbf{Conservation-aware transport of pH in the carbonate system.} (a) Passive tracers are transported as conserved physical inventories and therefore obey conservation laws. (b) In contrast, pH is an equilibrium-dependent chemical coordinate coupled to carbonate variables such as DIC and $p\mathrm{CO}_2$. (c) Treating pH as a passive tracer preserves only the transported pH field and does not generally conserve the associated carbon inventory, while introducing a latent conserved state that decodes to both pH and DIC enables transport consistent with DIC conservation.}
    \label{fig:active}
\vspace{-15pt}
\end{figure}

\section{Problem Formulation: Active Tracer Reconstruction}

\textbf{Notations.}
Let \(\Omega\subset S^2\) denote the ocean surface, with spatial coordinate \(x\) and time \(t\in[0,T]\). For a field \(q\), write \(\dot q=\partial q/\partial t\), and use \(\nabla\cdot\), \(\nabla\), and \(\nabla^2\) for spherical divergence, gradient, and Laplacian. Let \(u\) be the surface velocity, \(K\) the diffusion tensor, \(h\) the surface pH, \(d\) dissolved inorganic carbon (DIC), and \(c\) surface \(p\mathrm{CO}_2\). We write \(M[r](t)=\int_\Omega r(x,t)\,dx\) for spatial mass, and \(\mathcal I=\mathcal Z\times\mathcal T\) for the observed space--time sites.

\textbf{Active vs. passive tracers: chemical consistency.}
As illustrated in Figure~\ref{fig:active}, pH is not a standalone transported scalar. It is a chemical coordinate of the carbonate system. {Carbonate equilibrium gives two deterministic maps: \(\Gamma_D(h,c)\) returns the DIC implied by pH and \(p\mathrm{CO}_2\), while \(\Gamma_C(h,d)\) returns the \(p\mathrm{CO}_2\) implied by pH and DIC. }Therefore, a reconstructed pH field \(\hat h\) is chemically meaningful only when \(\Gamma_D(\hat h,c)\approx d\) and \(\Gamma_C(\hat h,d)\approx c\). We measure this disagreement on the observation set \(\mathcal I\):
\begin{equation}\label{eq::chemical}
\mathcal E_{\mathrm{chem}}(\hat h)=\|\Gamma_D(\hat h,c)-d\|_{L^2(\mathcal I)}^2+\beta\|\Gamma_C(\hat h,d)-c\|_{L^2(\mathcal I)}^2.
\end{equation}
Because carbonate equilibrium is nonlinear, low pH RMSE does not guarantee low DIC or \(p\mathrm{CO}_2\) error. The reconstructed pH must therefore be checked inside the carbonate system.

\textbf{Active vs. passive tracers: physical consistency.}
Figure~\ref{fig:active}~(c) shows the corresponding conservation mismatch: the source-free conserved inventory is DIC, not pH. Under no-flux boundaries, source-free DIC transport should satisfy \(M[d](t)=M[d](t_0)\). Therefore, for a flow-only pH prediction \(\hat h_{\mathrm{flow}}\), the relevant constraint is \(M[\Gamma_D(\hat h_{\mathrm{flow}},c)](t)=M[\Gamma_D(\hat h_{\mathrm{flow}},c)](t_0)\), not conservation of pH or \(H^+\). The full carbonate system may change through biology, air--sea exchange, vertical mixing, and unresolved reactions, but transport alone should not create or destroy carbon.

\textbf{Reconstruction problem.}
Given sparse pH observations \(\mathcal O_h\) on \(\mathcal I\), velocity \(u\), reference \(p\mathrm{CO}_2\) \(c\), and reference DIC \(d\), we reconstruct \(h\) by minimizing \(\mathcal L_{\mathrm{pH}}+\lambda_{\mathrm{chem}}\mathcal E_{\mathrm{chem}}+\lambda_{\mathrm{phys}}\mathcal E_{\mathrm{phys}}\), where \(\mathcal L_{\mathrm{pH}}=\|\hat h-h\|_{L^2(\mathcal I)}^2\) and \(\mathcal E_{\mathrm{phys}}=\int_{t_0}^{T}|M[\Gamma_D(\hat h_{\mathrm{flow}},c)](t)-M[\Gamma_D(\hat h_{\mathrm{flow}},c)](t_0)|^2dt\). This formulation separates the predicted variable from the conserved inventory: pH is reconstructed, but the carbon state implied by pH and \(p\mathrm{CO}_2\) must remain chemically and physically valid.

\section{Methodology}

\subsection{A Naive Passive-Tracer Baseline for pH Reconstruction}

\textbf{Naive passive pH reconstruction.}
A natural baseline is to treat pH as a transported scalar. Let $F_0$ be a data-driven initializer and let $\xi$ denote auxiliary fields. Given sparse pH observations $\mathcal O_h$, velocity $u$, reference $p\mathrm{CO}_2$ $c$, and reference DIC $d$, the model first estimates the initial field by $\hat h(\cdot,t_0)=F_0(\mathcal O_h,u,c,d,\xi)$. It then evolves $\hat h$ by the passive advection--diffusion equation $\dot{\hat h}+\nabla\cdot(u\hat h)-\nabla\cdot(K\nabla\hat h)=s_h$, where $K$ and $s_h$ are learned. The model is trained only by scalar pH error, $\mathcal L_{\mathrm{pH}}=\|\hat h-h\|_2^2$. This baseline is stronger than frame-wise inpainting because it uses temporal propagation, velocity fields, and a PDE-based inductive bias.

{\textbf{Why it fails.}
The failure is structural: the model transports pH, but source-free carbon conservation is about DIC. Let \(\hat h_{\mathrm{flow}}\) be the flow-only prediction and \(D_{\mathrm{flow}}=\Gamma_D(\hat h_{\mathrm{flow}},c)\) be its implied DIC. Even if passive transport conserves the transported scalar, \(M[\hat h_{\mathrm{flow}}](t)=M[\hat h_{\mathrm{flow}}](t_0)\) does not imply \(M[D_{\mathrm{flow}}](t)=M[D_{\mathrm{flow}}](t_0)\). Thus the passive baseline can produce smooth, low-RMSE pH fields while still violating carbonate closure and the DIC budget. Appendix~\ref{app:passive-counterexample} makes this failure formal: even in a source-free, no-flux setting, passive pH transport conserves \(M[\hat h_{\mathrm{flow}}]\) but generically does not conserve \(M[\Gamma_D(\hat h_{\mathrm{flow}},c)]\), because the carbonate map \(\Gamma_D(\cdot,c)\) is nonlinear.}

\begin{figure}
\centering
\captionsetup{font=footnotesize,skip=2pt}
\includegraphics[width=\linewidth]{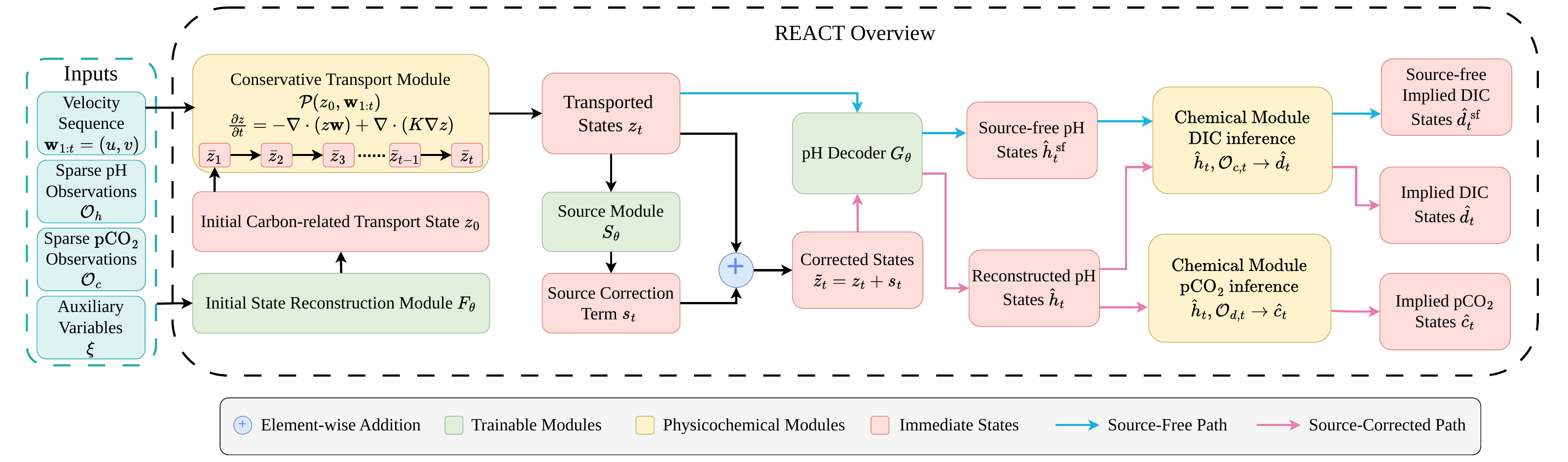}
\caption{\small\textbf{Overview of REACT}. REACT initializes a carbon-related state from sparse observations and physical drivers, propagates it through a conservative transport module, adds learned active source corrections, and decodes the corrected state into pH. The reconstructed pH is checked by carbonate chemistry, while the source-free branch is constrained to conserve represented DIC mass.}
\label{fig:structure}
\end{figure}

\subsection{REACT: \underline{RE}constructing \underline{A}ctive \underline{C}arbonate \underline{T}racers}\label{sec::react}

\textbf{Model overview.}
We propose \textbf{REACT}, short for \textbf{RE}constructing \textbf{A}ctive \textbf{C}arbonate \textbf{T}racers. As shown in Figure~\ref{fig:structure}, REACT follows a carbon-first pipeline. It takes velocity sequences, sparse pH observations, sparse \(p\mathrm{CO}_2\) observations, and auxiliary variables as inputs. It first reconstructs an initial carbon-related transport state \(z_0\), propagates it by a conservative PDE to obtain transported states \(z_t\), adds an active source correction \(s_t\) to form corrected states \(\tilde z_t=z_t+s_t\), and decodes the corrected states into reconstructed pH \(\hat h_t\). A chemical module then maps \(\hat h_t\) back to implied DIC through carbonate equilibrium. The key difference from passive-tracer reconstruction is therefore not the use of a PDE, but what happens around the PDE: the transported state is carbon-related, the residual is an active carbon-cycle source, and the output is checked by carbonate chemistry. 

\textbf{Latent carbonate-state transport and active correction.}
REACT consists of three coupled components. (1) \textit{Latent carbonate-state initialization}: REACT estimates an initial latent carbonate state \(z_0=F_\theta(\mathcal O_h,\mathcal O_c,u,\xi)\), where \(\mathcal O_h\) and \(\mathcal O_c\) are sparse pH and \(p\mathrm{CO}_2\) observations, and \(\xi\) denotes auxiliary fields. The state $z_0$ is not a pH image or raw DIC field; it is a latent coordinate on which conservative carbonate transport is imposed. (2) \textit{Conservative transport}: starting from \(z(\cdot,t_0)=z_0\), REACT evolves \(z\) by \(\partial_t z+\nabla\cdot(uz)-\nabla\cdot(K_\theta\nabla z)=0\), where \(K_\theta\) is learned and \(z_t=z(\cdot,t)\). This gives the source-free branch for conservative redistribution. (3) \textit{Active correction and pH decoding}: because the surface carbonate system also changes through air--sea exchange, biological activity, vertical mixing, and unresolved reactions, REACT predicts an active correction \(s_t=S_\theta(z_t,u,c,\xi)\), forms \(\tilde z_t=z_t+s_t\), and decodes pH by \(\hat h_t=G_\theta(\tilde z_t,c,\xi)\). Thus the PDE branch explains conservative transport, while the source module explains non-conservative carbonate-system changes.

\textbf{Represented DIC and source-free conservation.}
Since \(z_t\) may be latent, we do not require it to equal DIC in raw units. We require only that it has a DIC readout after decoding and carbonate equilibrium, defined by \(\mathcal D_\theta(z_t,c_t)=\Gamma_D(G_\theta(z_t,c_t),c_t)\). In a DIC-coordinate implementation, \(\mathcal D_\theta(z_t,c_t)\) reduces to \(z_t\). The corrected state \(\tilde z_t\) need not conserve mass because it contains active sources, but the source-free branch \(z_t\) should conserve represented DIC. We therefore impose \(\mathcal L_{\mathrm{phys}}=\int_{t_0}^{T}|M[\mathcal D_\theta(z_t,c_t)]-M[\mathcal D_\theta(z_0,c_{t_0})]|^2dt\). This placement is crucial: conservation before the source module forces the PDE branch to carry conservative carbon transport, instead of allowing the source module to hide transport errors.

\textbf{Training objective.}
REACT is trained with pH supervision, carbonate closure, and source-free DIC conservation. The pH loss is \(\mathcal L_{\mathrm{pH}}=\|\hat h-h\|_{L^2(\mathcal I)}^2\), and the chemical loss is \(\mathcal L_{\mathrm{chem}}=\mathcal E_{\mathrm{chem}}(\hat h)\), which checks whether reconstructed pH implies the correct DIC when paired with \(p\mathrm{CO}_2\), and the correct \(p\mathrm{CO}_2\) when paired with DIC. The total loss is \(\mathcal L_{\mathrm{total}}=\mathcal L_{\mathrm{pH}}+\lambda_{\mathrm{chem}}\mathcal L_{\mathrm{chem}}+\lambda_{\mathrm{phys}}\mathcal L_{\mathrm{phys}}\). When the losses are small, the prediction is not only accurate in pH, but also chemically closed and physically conservative in the source-free branch.

\textbf{Advantage over the passive baseline.}
REACT follows the PDE-plus-correction principle used in  models such as NUTS~\citep{nuts}: both use dynamical propagation and learned residual modules rather than frame-wise inpainting. The difference is the tracer. For passive nutrients, the reconstructed variable is also the transported inventory, so PDE propagation can act directly on the target field. For carbonate reconstruction, the target is pH but the conserved inventory is DIC. A passive pH model therefore has the right dynamical tool but the wrong state variable. REACT keeps the useful part of PDE-based reconstruction while changing the transported state from pH to represented carbon, and adds carbonate closure to ensure that the decoded pH remains chemically meaningful.

\textbf{Relation to existing reconstruction models.}
Existing reconstruction methods can be viewed along two axes, with carbonate-aware reconstruction exposing a third missing axis. Frame-wise interpolation or inpainting ignores transport and treats each map independently~\citep{conkright2002world-OI,ronneberger2015unet,vit,zhao2021large-CoModGAN,repaint}. PDE-based reconstruction introduces temporal dynamics, but usually assumes the reconstructed variable is itself the transported tracer~\citep{courtier1994strategy-4D-var,nerger2008improving-KalmanFilter,brasseur1991application-VIM,fno,marble,luo2024inr,beauchamp20234dvarnet,nuts}. These methods do not directly address pH as an active carbonate-system variable: chemical closure is not part of the reconstruction target, and conservative transport is not explicitly separated from active carbon-cycle sources. REACT takes one step further: it places the PDE on a latent carbon state and evaluates decoded pH through carbonate closure.

\subsection{Theoretical Analysis: Why REACT Uses a Carbon State}
\label{sec:react-theory}

The theoretical question is whether the hidden state \(z_t\) carries a physical certificate. REACT does not require \(z_t\) to equal raw DIC, since a latent model may represent the same carbon inventory in another coordinate system. What matters is the DIC readout induced by the decoder and carbonate chemistry: \(\mathcal D_\theta(z_t,c_t)=\Gamma_D(G_\theta(z_t,c_t),c_t)\). We therefore interpret \(z_t\) as a carbon coordinate and \(\mathcal D_\theta(z_t,c_t)\) as its represented DIC.


\textbf{Theorem 1 (Informal; carbon-state certificate).}
Under the standard identifiability and realizability conditions stated in Appendix~\ref{app:react-theory}, REACT gives the following certificate. If \(\mathcal L_{\mathrm{chem}}=0\) and \(\mathcal L_{\mathrm{phys}}=0\), then the reconstructed pH closes carbonate chemistry, i.e., \(\Gamma_D(\hat h,c)=d\) and \(\Gamma_C(\hat h,d)=c\), and the source-free branch conserves represented DIC mass, i.e., \(M[\mathcal D_\theta(z,c)](t)=M[\mathcal D_\theta(z,c)](t_0)\). Moreover, when the true process decomposes into source-free DIC transport plus an active carbon residual, REACT can realize zero chemical and physical loss. A passive pH-transport model cannot generally satisfy the same certificate: even if it conserves source-free pH mass, it can change the DIC mass implied by carbonate chemistry.

\textbf{Interpretation.}
The result explains why REACT transports a carbon state rather than pH itself. For passive tracers, the prediction target is also the conserved inventory. For carbonate tracers, this identity breaks: pH is the target, but DIC is the conserved inventory. Since the carbonate map \(D(h)=\Gamma_D(h,c)\) is nonlinear, conserving pH does not generally conserve DIC. REACT resolves this mismatch by placing the PDE on a carbon state, assigning active carbon-cycle changes to the source module, and decoding pH through carbonate closure. The full proof and the passive-baseline counterexample are given in Appendix~\ref{app:react-theory}.

\section{Experiment}

In this section, we evaluate REACT through three questions:

\textbf{RQ1.} Does REACT improve global sea surface pH reconstruction across temporal scales and real observations? Does it also improve carbonate-system consistency?

\textbf{RQ2.} Do the mass and chemical constraints make the reconstruction physically and chemically valid? How much do they reduce DIC drift and carbonate error?

\textbf{RQ3.} How do different design choices, such as auxiliary variable inputs, model depth, and reconstruction window length, affect model performance?

\subsection{Experimental Setup}
In this subsection, we describe the datasets, tasks, baselines and metrics. Additional details are provided in Appendix~\ref{app:experiment_details}.

\textbf{Data.}
We use two types of datasets in this work: a gridded simulation dataset derived from MOM6--COBALTv2~\citep{MOM6} and a sparse observational dataset in the World Ocean Database (WOD)~\citep{WOD} and the Surface Ocean CO$_2$ Atlas (SOCAT)~\citep{SOCAT}. 
\textbf{Simulation Dataset.} The simulation data come from daily MOM6--COBALTv2 outputs from 1959 to 2022. All variables are regridded to a uniform $1^\circ \times 1^\circ$ global grid using bilinear interpolation, resulting in a spatial resolution of $180 \times 360$. We aggregate the daily outputs into weekly and monthly products, yielding two cross-scale tasks: month-to-week and week-to-day reconstruction. The dataset provides sea surface pH as the reconstruction target, together with nine auxiliary variables used as model inputs: sea surface CO$_2$ partial pressure (\(p\mathrm{CO}_2^s\)), atmospheric CO$_2$ partial pressure (\(p\mathrm{CO}_2^a\)), chlorophyll concentration (Chl.), sea surface height (SSH), mixed layer depth (MLD), salinity (Salt), temperature (Temp.), and horizontal velocity components ($u$ and $v$). 
\textbf{Observational data.} The pH observations are from the WOD, while \(p\mathrm{CO}_2^s\) observations are from the SOCAT, both  from 1982 to 2022. The records are highly sparse: valid entries cover only \(0.11\%\) of the pH grid and \(1.67\%\) of the \(p\mathrm{CO}_2^s\) grid. All observations are mapped onto the same spatiotemporal grid as MOM6.

\textbf{Tasks.}
We evaluate the proposed model on two types of tasks: simulation-based and observation-based reconstruction.
\textbf{Simulation-based reconstruction.}
We consider two cross-temporal-scale settings, monthly-to-weekly and weekly-to-daily sea surface pH reconstruction. Given sparse coarse-scale pH and \(p\mathrm{CO}_2^s\) observations, together with target-scale auxiliary variables, the model reconstructs the full target-scale pH field. We sample the full MOM6 labels at $0.11\%$, $1\%$, and $10\%$, where $0.11\%$ matches the sparsity of WOD pH observations and $10\%$ follows prior sparse-field reconstruction settings.
\begin{wraptable}{r}{0.52\textwidth}
\vspace{-6pt}
\centering
\tiny
\setlength{\tabcolsep}{2.5pt}
\renewcommand{\arraystretch}{1.05}
\caption{Overview of dataset divisions by year.}
\label{tab:data_split}
\vspace{-4pt}
\resizebox{0.52\textwidth}{!}{
\begin{tabular}{cccc}
\toprule
\textbf{Task} & \textbf{Train} & \textbf{Val.} & \textbf{Test} \\
\midrule
\textbf{Month$\to$Week} & 1971--2010 & 1965--1970; 2011--2016 & 1959--1964; 2017--2022 \\
\textbf{Week$\to$Day}   & 1988--2003 & 1987; 2004 & 1986; 2005 \\
\textbf{WOD}   & 1982--2018 & 2019--2020 & 2021--2022 \\
\bottomrule
\end{tabular}
}
\vspace{-10pt}
\end{wraptable}
\textbf{Observation-based reconstruction.}
We use WOD pH and SOCAT \(p\mathrm{CO}_2^s\) as sparse constraints, together with MOM6 auxiliary variables, to reconstruct sea surface pH under real observational sparsity.
All tasks use chronological splits for training, validation, and testing, as summarized in Table~\ref{tab:data_split}. 

\textbf{Baselines.}
We compare REACT with seven classes of baselines:
(1) spatial interpolation methods, including Kriging with exponential and spherical variogram models~\citep{kriging};
(2) convolutional networks, represented by U-Net~\citep{ronneberger2015unet};
(3) Transformer-based models, represented by ViT~\citep{vit};
(4) neural operators, represented by FNO~\citep{fno};
(5) implicit continuous field reconstruction models, represented by Marble~\citep{marble};
(6) climate foundation models, including Prithvi WxC~\citep{prithvi} and AtmoRep~\citep{atmorep};
(7) physically guided reconstruction models, including 4DVarNet~\citep{beauchamp20234dvarnet} and NUTS~\citep{nuts}.
Kriging is applied independently at each target time slice using only sparse pH observations. All other baselines use the same temporal reconstruction windows, data splits, sampling ratios, and metrics, with sparse observations, auxiliary variables, and flow fields as inputs. Marble uses temporal observations but does not use explicit flow or auxiliary variables. 

\textbf{Metrics.}
We evaluate model performance using four metrics.
\textbf{pH} is the latitude-weighted NRMSE of sea surface pH, where the latitude weighting accounts for unequal grid-cell areas. NRMSE is obtained by normalizing RMSE with the mean value of the corresponding target variable.
\textbf{Chem.} measures carbonate closure, defined as the average of normalized \(p\mathrm{CO}_2^s\) error and normalized DIC error.
\textbf{Drift$_{\mathrm{local}}$} measures short term mass conservation behavior by computing the DIC inventory drift within each reconstruction segment.
\textbf{Drift$_{\mathrm{global}}$} measures long range mass conservation behavior by computing the accumulated DIC inventory drift over continuous long term reconstruction periods.
The detailed computation of all metrics is provided in Appendix~\ref{app:metrics}.

\subsection{Main Results (RQ1)}

\newcommand{\subpm}[1]{$_{\scalebox{0.70}{$\scriptstyle\pm#1$}}$}
\newcommand{\best}[1]{\textbf{#1}}
\newcommand{\second}[1]{\underline{#1}}
\newcommand{\oursrow}{\rowcolor{gray!12}}
\newcommand{\gainrow}{\rowcolor{gray!18}}

\begin{table}[t]
\scriptsize
\centering
\caption{\textbf{Main results.} Sea surface pH reconstruction on Month\(\to\)Week and Week\(\to\)Day tasks under different sampling ratios. pH NRMSE is reported in units of \(10^{-4}\), and Chem. is reported in units of \(10^{-2}\). Best and second-best results are shown in \textbf{bold} and \underline{underline}, respectively. Subscript values denote standard errors over three independent sparse-observation samplings. \textit{Params} denotes model size. \textit{U-Net$^{2}$} denotes a two-level U-Net variant. \textit{Gain} denotes relative improvement over the strongest baseline. Decomposed \(p\mathrm{CO}_2^s\) and DIC errors are reported in Appendix Table~\ref{tab:main_results_pco2_dic}.}
\label{tab:main_results}

\setlength{\tabcolsep}{1.5pt}
\renewcommand{\arraystretch}{1}
\begin{tabular}{c c cc cc cc cc cc cc}
\toprule
\multirow{3}{*}{\textbf{Methods}} &
\multirow{3}{*}{\textbf{Params}} &
\multicolumn{6}{c}{\textbf{Month\(\to\)Week}} &
\multicolumn{6}{c}{\textbf{Week\(\to\)Day}} \\
\cmidrule(lr){3-8}\cmidrule(lr){9-14}
& &
\multicolumn{2}{c}{0.11\%} & \multicolumn{2}{c}{1\%} & \multicolumn{2}{c}{10\%} &
\multicolumn{2}{c}{0.11\%} & \multicolumn{2}{c}{1\%} & \multicolumn{2}{c}{10\%} \\
\cmidrule(lr){3-4}\cmidrule(lr){5-6}\cmidrule(lr){7-8}
\cmidrule(lr){9-10}\cmidrule(lr){11-12}\cmidrule(lr){13-14}
& &
pH & Chem. & pH & Chem. & pH & Chem. &
pH & Chem. & pH & Chem. & pH & Chem. \\
\midrule

Krig-Exp & -- & 50.6\subpm{0.91} & 10.1\subpm{0.34} & 37.0\subpm{0.84} & 7.06\subpm{0.26} & 23.3\subpm{0.72} & 4.56\subpm{0.23} & 51.2\subpm{0.86} & 9.95\subpm{0.29} & 37.1\subpm{0.82} & 6.94\subpm{0.27} & 22.2\subpm{0.68} & 4.22\subpm{0.17} \\
Krig-Sph & -- & 50.9\subpm{0.89} & 10.2\subpm{0.20} & 41.2\subpm{0.88} & 7.83\subpm{0.27} & 33.8\subpm{0.81} & 6.11\subpm{0.28} & 52.2\subpm{0.83} & 10.2\subpm{0.33} & 39.4\subpm{0.84} & 7.35\subpm{0.30} & 26.9\subpm{0.66} & 5.04\subpm{0.19} \\
4D-Var. & 0.27M & 6.92\subpm{0.03} & 2.09\subpm{0.03} & 7.14\subpm{0.04} & \second{2.00\subpm{0.02}} & \second{5.99\subpm{0.06}} & \second{1.90\subpm{0.04}} & 5.93\subpm{0.07} & 1.86\subpm{0.03} & 6.02\subpm{0.05} & 1.88\subpm{0.03} & 7.40\subpm{0.05} & 2.05\subpm{0.04} \\
Marble & 0.83M & 25.7\subpm{0.03} & 5.65\subpm{0.10} & 25.8\subpm{0.03} & 5.71\subpm{0.07} & 25.2\subpm{0.02} & 5.51\subpm{0.07} & 26.9\subpm{0.04} & 5.51\subpm{0.06} & 26.9\subpm{0.02} & 5.55\subpm{0.03} & 27.0\subpm{0.03} & 5.51\subpm{0.05} \\
FNO & 8.42M & 16.7\subpm{0.07} & 3.88\subpm{0.07} & 16.3\subpm{0.04} & 3.81\subpm{0.02} & 16.6\subpm{0.05} & 3.82\subpm{0.02} & 7.34\subpm{0.05} & 2.07\subpm{0.04} & 8.81\subpm{0.06} & 2.29\subpm{0.02} & 8.67\subpm{0.05} & 2.24\subpm{0.02} \\
U-Net & 31.1M & 7.78\subpm{0.04} & 2.17\subpm{0.08} & 7.56\subpm{0.05} & 2.14\subpm{0.01} & 7.40\subpm{0.04} & 2.12\subpm{0.03} & 5.91\subpm{0.06} & 1.86\subpm{0.05} & 5.81\subpm{0.05} & 1.85\subpm{0.03} & 5.80\subpm{0.05} & 1.84\subpm{0.03} \\
U-Net$^{2}$ & 1.96M & 7.71\subpm{0.13} & 2.21\subpm{0.03} & 7.78\subpm{0.09} & 2.22\subpm{0.02} & 7.22\subpm{0.12} & 2.11\subpm{0.03} & 7.21\subpm{0.14} & 2.02\subpm{0.04} & 7.29\subpm{0.12} & 2.06\subpm{0.02} & 6.81\subpm{0.08} & 1.93\subpm{0.02} \\
ViT & 0.11B & 9.56\subpm{0.10} & 2.51\subpm{0.01} & 9.47\subpm{0.04} & 2.54\subpm{0.01} & 9.26\subpm{0.06} & 2.50\subpm{0.01} & 6.31\subpm{0.06} & \second{1.66\subpm{0.03}} & 6.26\subpm{0.04} & 2.18\subpm{0.03} & 5.85\subpm{0.06} & 1.88\subpm{0.02} \\
NUTS & 0.15B & \second{6.74\subpm{0.06}} & \second{1.95\subpm{0.10}} & \second{6.88\subpm{0.08}} & 2.06\subpm{0.12} & 6.87\subpm{0.05} & 2.30\subpm{0.09} & \second{5.08\subpm{0.09}} & 1.69\subpm{0.07} & \second{4.91\subpm{0.07}} & \second{1.65\subpm{0.05}} & \second{5.05\subpm{0.05}} & \second{1.77\subpm{0.04}} \\
AtmoR. & 0.80B & 14.6\subpm{0.16} & 3.52\subpm{0.04} & 16.2\subpm{0.19} & 3.95\subpm{0.04} & 11.0\subpm{0.17} & 2.86\subpm{0.03} & 9.10\subpm{0.12} & 2.35\subpm{0.05} & 7.55\subpm{0.14} & 2.07\subpm{0.03} & 7.60\subpm{0.11} & 2.05\subpm{0.01} \\
Prithvi & 2.37B & 16.0\subpm{0.52} & 3.48\subpm{0.07} & 16.1\subpm{0.47} & 3.45\subpm{0.03} & 14.2\subpm{0.59} & 3.16\subpm{0.05} & 13.3\subpm{0.29} & 2.75\subpm{0.04} & 13.4\subpm{0.36} & 2.75\subpm{0.02} & 12.8\subpm{0.44} & 2.73\subpm{0.07} \\

\midrule
\oursrow
\textbf{REACT} & 0.18B &
\best{5.84\subpm{0.04}} & \best{1.59\subpm{0.04}} &
\best{5.87\subpm{0.03}} & \best{1.52\subpm{0.02}} &
\best{5.81\subpm{0.02}} & \best{1.54\subpm{0.03}} &
\best{4.64\subpm{0.04}} & \best{1.35\subpm{0.03}} &
\best{4.62\subpm{0.03}} & \best{1.34\subpm{0.02}} &
\best{4.54\subpm{0.02}} & \best{1.32\subpm{0.02}} \\
\gainrow
\textbf{\textit{Gain}} & -- &
13.4\% & 18.5\% &
14.7\% & 24.0\% &
3.00\% & 18.9\% &
8.66\% & 18.7\% &
5.90\% & 18.8\% &
10.1\% & 25.4\% \\
\bottomrule
\end{tabular}%
\vspace{-0.3cm}
\end{table}

In this subsection, we present the main results.

\begin{wrapfigure}{r}{0.40\linewidth}
\vspace{-0.4cm}
\centering
\captionsetup{font=footnotesize,skip=2pt}
\includegraphics[width=0.95\linewidth]
{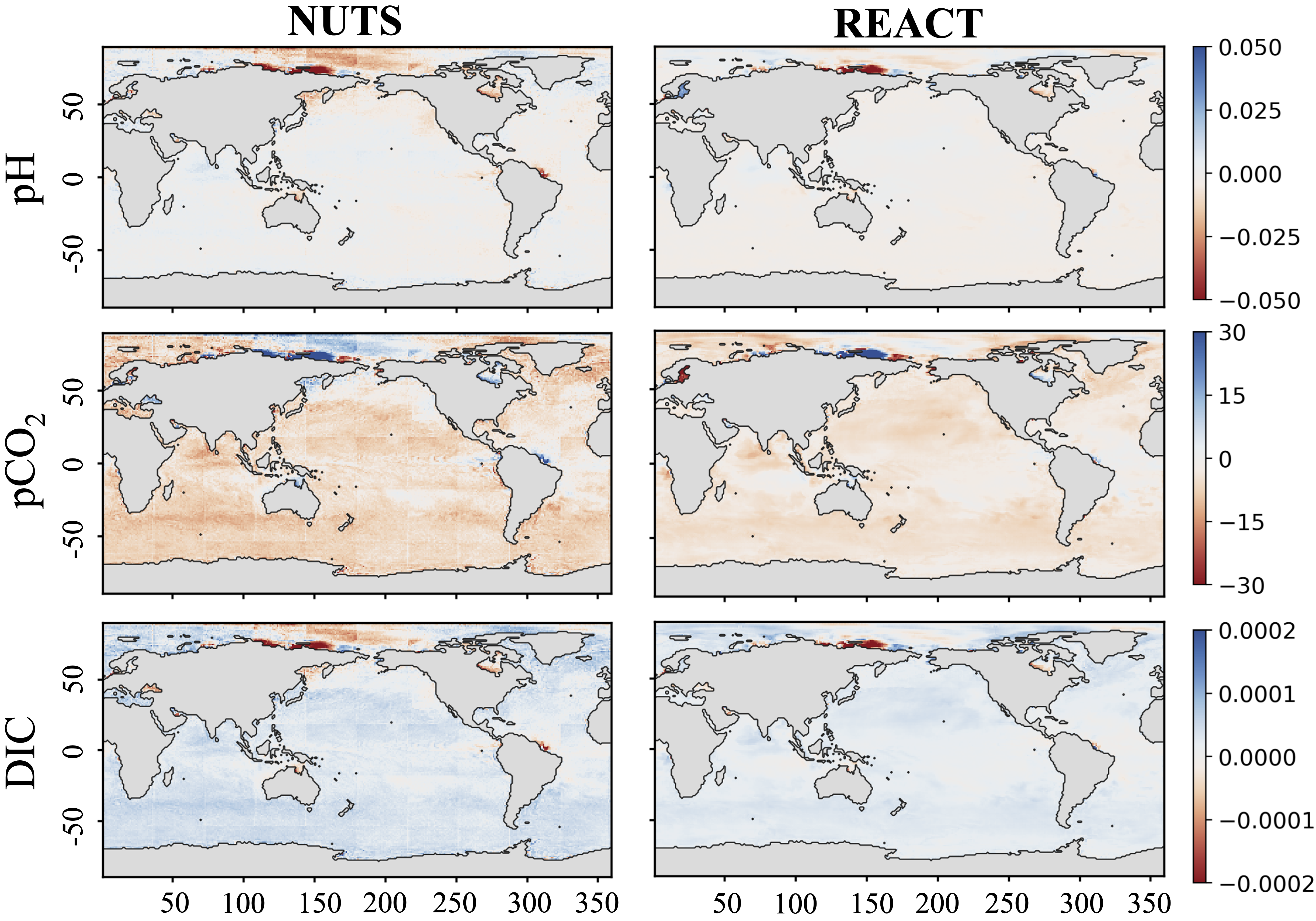}
\vspace{0.2cm}
\caption{\textbf{Qualitative error maps.} Spatial reconstruction errors on MOM6, computed as prediction minus ground truth.  Additional qualitative results are provided in Appendix~\ref{app:additional_qualitative_results}.}
\label{fig:qualitative}
\vspace{-0.4cm}
\end{wrapfigure}

\textbf{Obs 1: REACT improves both pH accuracy and carbonate closure across temporal scales.}
As shown in Table~\ref{tab:main_results}, REACT is the best method on both month-to-week and week-to-day reconstruction, improving both pH NRMSE and Chem.
In terms of pH NRMSE, REACT reduces the error by 3.00\%--14.7\% on monthly-to-weekly reconstruction and by 5.90\%--10.1\% on weekly-to-daily reconstruction.
This suggests that REACT better controls error propagation when refining coarse temporal inputs into finer pH fields.
Figure~\ref{fig:qualitative} further shows smaller spatial errors than NUTS, especially in the Arctic and Pacific regions.
The gains in Chem. are even larger.
REACT improves Chem. by 18.5\%--24.0\% on the monthly-to-weekly task and by 18.7\%--25.4\% on the weekly-to-daily task.
This shows that REACT not only reduces pointwise pH error, but also produces pH fields that are more consistent with the associated \(p\mathrm{CO}_2^s\) and DIC fields.
The gain comes from modeling pH as an active carbonate tracer rather than a standalone scalar field.

\begin{wrapfigure}{r}{0.37\linewidth}
\vspace{-0.5cm}
    \centering
    \captionsetup{font=footnotesize,skip=2pt}

    \includegraphics[width=0.98\linewidth]{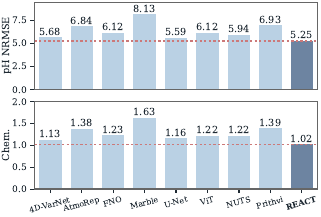}

    \caption{\textbf{Real-observation results.}  pH NRMSE and Chem. measure are reported in units of \(10^{-3}\) and \(10^{-1}\), respectively. Lower is better. Decomposed \(p\mathrm{CO}_2^s\) and DIC results are provided in Appendix~\ref{app:wod_decomposed_results}.}
    \label{fig:wod_real_results}
    \vspace{-1.0em}
\end{wrapfigure}

\textbf{Obs 2: REACT transfers to real sparse observations.}
We evaluate real sparse reconstruction using WOD pH and SOCAT \(p\mathrm{CO}_2^s\). We train on \(75\%\) of observed sites and test on the remaining \(25\%\), which measures generalization under irregular real-world sampling.
As shown in Figure~\ref{fig:wod_real_results}, REACT achieves a pH NRMSE of \(5.25\times10^{-3}\), improving over 4DVarNet, the strongest baseline, which obtains \(5.68\times10^{-3}\).
This corresponds to an error reduction of approximately 7.57\%.
Meanwhile, REACT achieves a Chem. score of 0.102, also outperforming 4DVarNet with 0.113, corresponding to an improvement of approximately 9.73\%.
Thus, the advantage of REACT is not limited to regular simulation grids; it also holds under sparse and irregular real observations.

\subsection{Component Analysis (RQ2)}

We isolate three active-tracer components: source correction, DIC mass conservation and carbonate chemical closure. They test whether REACT can capture non-conservative pH changes, preserve DIC inventory, and close the carbonate system beyond pointwise pH fitting. Table~\ref{tab:module_ablation} reports the ablations, Figure~\ref{figure:mass_conservation_drift} shows DIC drift and Figure~\ref{figure:uncertainty_ph} evaluates uncertainty calibration.

\begin{table*}[t]
\centering
\scriptsize
\setlength{\tabcolsep}{2pt}
\renewcommand{\arraystretch}{1}
\caption{\textbf{Component analysis.} Ablation of the source module, DIC mass constraint, and carbonate chemical constraint. We report results for REACT on both reconstruction tasks and repeat the constraint ablation on U-Net for month-to-week reconstruction at \(0.11\%\) sampling. pH NRMSE is in units of \(10^{-4}\), Chem. is in units of \(10^{-2}\), and drift is reported in per mille. \textit{Div.} denotes divergent long-range drift. Decomposed \(p\mathrm{CO}_2^s\) and DIC errors are provided in Appendix Table~\ref{tab:loss_ablation_pco2_dic}.}

\label{tab:module_ablation}
\begin{tabular}{c cc cccc cccc cccc}
\toprule
&
\multicolumn{2}{c}{} &
\multicolumn{4}{c}{REACT (Month$\to$Week)} &
\multicolumn{4}{c}{REACT (Week$\to$Day)} &
\multicolumn{4}{c}{U-Net (Month$\to$Week)} \\
\cmidrule(lr){4-7}\cmidrule(lr){8-11}\cmidrule(lr){12-15}
Source &
$\mathcal{L}_{\mathrm{mass}}$ &
$\mathcal{L}_{\mathrm{chem}}$ &
pH &
Chem. &
Drift$_{\mathrm{local}}$ &
Drift$_{\mathrm{global}}$ &
pH &
Chem. &
Drift$_{\mathrm{local}}$ &
Drift$_{\mathrm{global}}$ &
pH &
Chem. &
Drift$_{\mathrm{local}}$ &
Drift$_{\mathrm{global}}$ \\
\midrule
--         & --         & --         & 40.8 & Div. & -- & -- & 11.8 & 2.73 & -- & -- & 7.74 & 2.12 & -- & -- \\
\checkmark & --         & --         & 5.79 & 1.93 & 8.84\textperthousand & Div. & 4.03 & 1.65 & 1.11\textperthousand & Div. & 7.75 & 2.14 & 21.2\textperthousand & 17.3\textperthousand \\
\checkmark & \checkmark & --         & 5.68 & 1.94 & 2.26\textperthousand & 9.73\textperthousand & 4.27 & 1.69 & 0.57\textperthousand & 7.41\textperthousand & 7.78 & 2.03 & 4.24\textperthousand & 3.34\textperthousand \\
\checkmark & \checkmark & \checkmark & 5.84 & 1.59 & 2.21\textperthousand & 10.3\textperthousand & 4.64 & 1.35 & 0.63\textperthousand & 7.68\textperthousand & 7.92 & 1.65 & 5.53\textperthousand & 4.88\textperthousand \\
\bottomrule
\end{tabular}
\vspace{-0.3cm}
\end{table*}

\begin{wrapfigure}[16]{r}{0.35\linewidth}
\vspace{-0.3cm}
    \centering
    \captionsetup{font=footnotesize,skip=2pt}
    \includegraphics[width=0.98\linewidth]{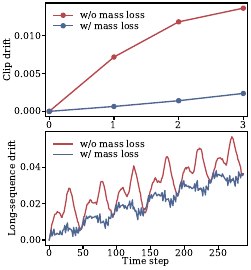}
    \caption{Effect of mass-conservation training on DIC drift.}
    \label{figure:mass_conservation_drift}

\end{wrapfigure}

\textbf{Obs 3: The source module is necessary for accurate pH reconstruction.}
Table~\ref{tab:module_ablation} shows that physical propagation alone is insufficient: it captures flow-driven transport but misses non-conservative effects from air--sea exchange, biology, and local model error.
With the source term correction module, the model can compensate for residual errors in the propagated intermediate field, leading to a clear reduction in pH error.
On month-to-week reconstruction at \(0.11\%\) sampling, the pH error drops from \(4.08\times10^{-3}\) to \(5.79\times10^{-4}\). This result suggests that sea surface pH reconstruction needs active correction beyond conservative transport.

\textbf{Obs 4: The mass constraint suppresses DIC drift without hurting pH accuracy.}
As shown in Table~\ref{tab:module_ablation} and Figure~\ref{figure:mass_conservation_drift}, since pH is not a mass-conservative variable, directly computing inventory drift on the pH field does not have a clear physical interpretation. Therefore, we evaluate mass conservation on the DIC field implied by carbonate chemistry. On month-to-week reconstruction, the mass constraint reduces local drift from $8.84$\textperthousand{} to $2.26$\textperthousand, and turns global drift from divergence into $9.73$\textperthousand. Meanwhile, the pH NRMSE slightly decreases from $5.79\times10^{-4}$ to $5.68\times10^{-4}$. A consistent pattern is observed in the U-Net ablation, where adding the mass constraint reduces local drift from $21.2$\textperthousand{} to $4.24$\textperthousand and global drift from $17.3$\textperthousand{} to $3.34$\textperthousand. This confirms that the constraint improves physical validity without sacrificing pH accuracy.

\begin{wrapfigure}{r}{0.35\linewidth}
\vspace{-0.4cm}
    \centering
    \captionsetup{font=footnotesize,skip=2pt}
    \includegraphics[
        width=0.98\linewidth,
        trim={0 20pt 0 20pt},
        clip
    ]{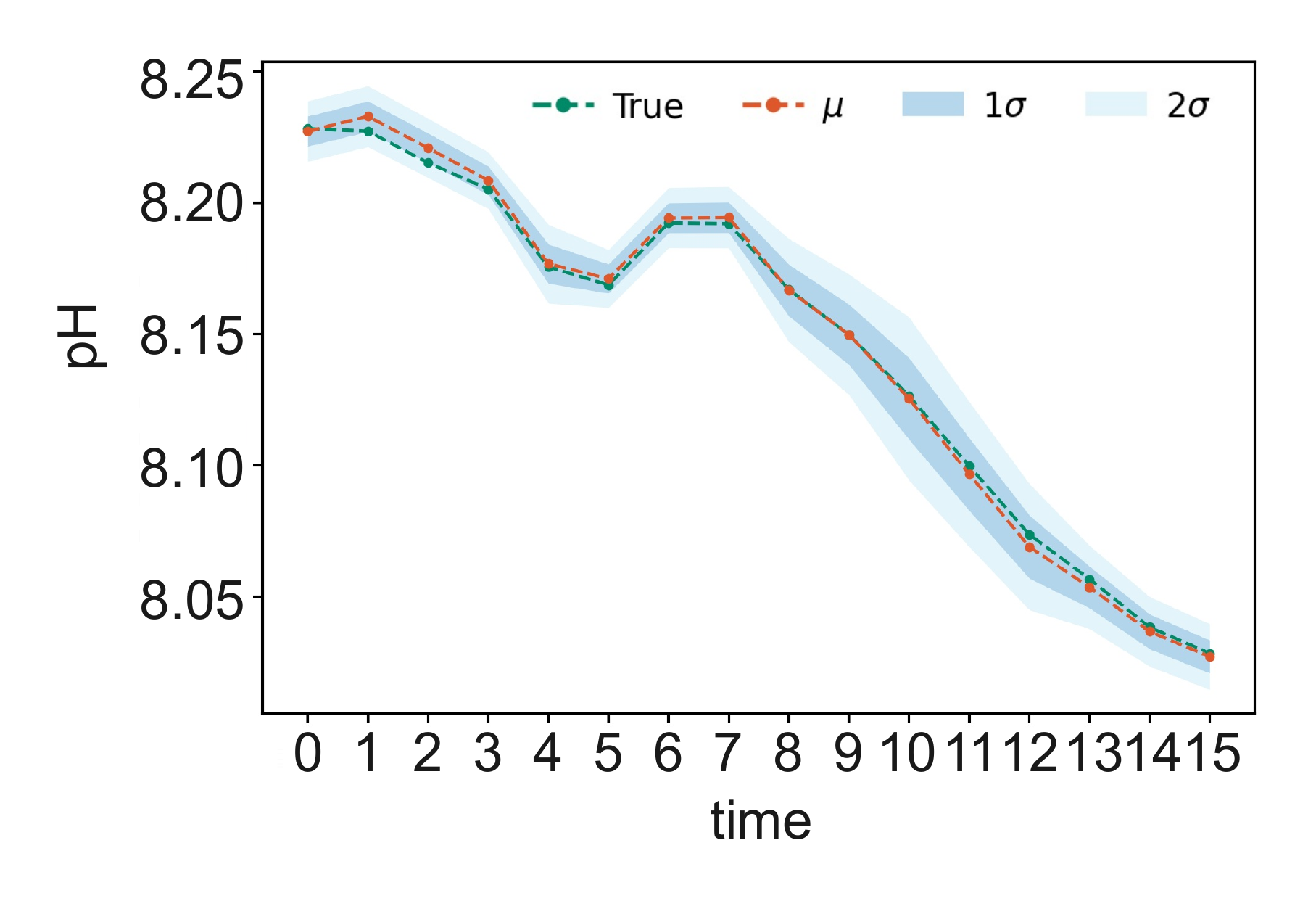}
    \caption{Uncertainty estimation of REACT.}
    \vspace{-0.3cm}
    \label{figure:uncertainty_ph}
\end{wrapfigure}
\textbf{Obs 5: The chemical constraint improves carbonate closure with a small accuracy trade-off.}
Table~\ref{tab:module_ablation} shows that optimizing pH alone can give accurate but chemically inconsistent predictions.
After introducing the chemical module, the overall chemical consistency metric improves noticeably, indicating improved consistency between the predicted pH field and the carbonate system. 
For example, on month-to-week reconstruction, the chemical metric improves from \(1.94\times10^{-2}\) to \(1.59\times10^{-2}\), while pH error changes only slightly from \(5.68\times10^{-4}\) to \(5.84\times10^{-4}\). The U-Net ablation exhibits the same behavior, with the chemical constraint improving Chem. from \(2.03\times10^{-2}\) to \(1.65\times10^{-2}\), accompanied by only a small increase in pH NRMSE. This suggests that the chemical constraint  trades a small amount of pointwise accuracy for a more physically valid carbonate state.

\textbf{Obs 6: REACT provides calibrated predictive uncertainty.}
Following ClimODE~\cite{verma2024climode}, we add a lightweight uncertainty head on top of the converged deterministic REACT model to estimate the predictive standard deviation, while keeping the predictive mean fixed. As shown in Figure~\ref{figure:uncertainty_ph}, the resulting uncertainty band covers most ground-truth pH trajectories and remains concentrated around the deterministic prediction. The selected checkpoint achieves two-standard-deviation coverage values of \(0.963\), \(0.971\), and \(0.958\) for pH, \(p\mathrm{CO}_2^s\), and DIC, respectively, close to the nominal Gaussian coverage of \(0.95\). This suggests that REACT can provide calibrated predictive uncertainty in addition to accurate deterministic reconstruction. Results for \(p\mathrm{CO}_2^s\) and DIC, together with training details, are provided in Appendix~\ref{app:uncertainty_results}.

\subsection{Ablation Study on Design Choices (RQ3)}

\begin{table*}[t]
\centering
\scriptsize
\setlength{\tabcolsep}{1.75pt}
\renewcommand{\arraystretch}{1.1}
\caption{\textbf{Ablation Study}. (a) Evaluation of reconstruction interval length; (b) analysis of state initializer depth; (c) analysis of pH decoder/source-module depth; (d) comparison of state initializer architectures; (e) quantification of the impact of auxiliary input variables. All experiments follow the default setting: Conducted on the Monthly-to-Weekly task under sampling ratio of $0.11\%$, using a 12-layer ViViT~\citep{vivit} as state initializer, pH decoder/source module of depth 4 and interval length set to 4 and 2 for Monthly-to-Weekly and Weekly-to-Daily, respectively. Details are provided in Table~\ref{tab:monthly_weekly_ablations_pco2_dic} and Table~\ref{tab:weekly_daily_ablation_full}}
\label{tab:ablations_choices}

\newlength{\gridsep}
\setlength{\gridsep}{6pt} 

\newlength{\gridcol}
\setlength{\gridcol}{\dimexpr(\textwidth-2\gridsep)/3\relax}

\begin{tabular}{@{}p{\gridcol}@{\hspace{\gridsep}}p{\gridcol}@{\hspace{\gridsep}}p{\gridcol}@{}}

\begin{minipage}[t]{\linewidth}\centering
{\footnotesize\textbf{(a) Interval Length.}}\\[3pt]
\begin{tabular}{c cccc}
\toprule
case & pH & Chem. & Drift$_{\mathrm{local}}$ & Drift$_{\mathrm{global}}$ \\
\midrule
1 & 5.85 & 1.57 & 2.10\textperthousand  & Div.  \\
2 & 6.52 & 1.51 & 2.90\textperthousand  & Div.  \\
4 & 5.84 & 1.59& 2.21\textperthousand  & 10.3\textperthousand  \\
8 & 6.09 & 1.58 & 2.44\textperthousand  & 10.2\textperthousand  \\
\bottomrule
\end{tabular}
\end{minipage}
&
\begin{minipage}[t]{\linewidth}\centering
{\scriptsize\textbf{(b) Depth of State Initializer.}}\\[3pt]
\begin{tabular}{c cccc}
\toprule
depth & pH & Chem. & Drift$_{\mathrm{local}}$ & Drift$_{\mathrm{global}}$ \\
\midrule
4  & 6.73 & 1.59 & 2.98\textperthousand  & 11.0\textperthousand  \\
6  & 6.16 & 1.60 & 2.30\textperthousand  & 8.62\textperthousand  \\
12 & 5.84 & 1.59 & 2.21\textperthousand  & 10.3\textperthousand  \\
24 & 5.98 & 1.59 & 8.75\textperthousand  & 19.5\textperthousand  \\
\bottomrule
\end{tabular}
\end{minipage}
&
\begin{minipage}[t]{\linewidth}\centering
{\scriptsize\textbf{(c) Depth of pH Decoder/Source Module.}}\\[3pt]
\begin{tabular}{c cccc}
\toprule
depth & pH & Chem. & Drift$_{\mathrm{local}}$ & Drift$_{\mathrm{global}}$ \\
\midrule
2 & 5.75 & 1.54 & 8.41\textperthousand  & 19.1\textperthousand  \\
3 & 5.88 & 1.54 & 7.91\textperthousand  & 18.4\textperthousand  \\
4 & 5.84 & 1.59 & 2.21\textperthousand  & 10.3\textperthousand \\
5 & 7.51 & 1.85 & 2.75\textperthousand  & 11.4\textperthousand \\
\bottomrule
\end{tabular}
\end{minipage}
\\[63pt]

\begin{minipage}[t]{\linewidth}\centering
{\footnotesize\textbf{(d) State Initializer Structure.}}\\[3pt]
\begin{tabular}{c cccc}
\toprule
case & pH & Chem. & Drift$_{\mathrm{local}}$ & Drift$_{\mathrm{global}}$ \\
\midrule
ViT    & 6.58 & 1.53 & 3.74\textperthousand & Div. \\
2D CNN & 5.96 & 1.67 & 8.81\textperthousand & 537\textperthousand \\
3D CNN & 36.6 & 462 & 8.44\textperthousand & Div. \\
ViViT  & 5.84 & 1.59 & 2.21\textperthousand & 10.3\textperthousand \\
\bottomrule
\end{tabular}
\end{minipage}
&
\multicolumn{2}{p{\dimexpr2\gridcol+\gridsep\relax}}{%
\begin{minipage}[t]{\linewidth}\centering
{\footnotesize\textbf{(e) Auxiliary Variables.}}\\[3pt]

{\setlength{\tabcolsep}{2.5pt}\renewcommand{\arraystretch}{0.96}
\begin{tabular}{c c ccccccccc}
\toprule
Task & Metric & \(p\mathrm{CO}_2^s\) & \(p\mathrm{CO}_2^a\) & Chl. & SSH & MLD & Salt & Temp. & $u$ & $v$ \\
\midrule
\multirow{2}{*}{Monthly to Weekly} 
 & pH & 21.7 & 6.98 & 5.97 & 6.18 & 5.93 & 5.95 & 6.05 & 6.17 & 6.32 \\
 & Chem. & 4.97 & 1.55 & 1.73 & 1.58 & 1.57 & 1.70 & 1.61 & 1.60 & 1.73 \\
\midrule
\multirow{2}{*}{Weekly to Daily} 
 & pH & 11.9 & 4.81 & 4.79 & 4.76 & 4.88 & 4.69 & 4.77 & 4.78 & 4.75 \\
 & Chem. & 2.57 & 1.33 & 1.30 & 1.26 & 1.25 & 1.40 & 1.29 & 1.26 & 1.31 \\
\bottomrule
\end{tabular}}
\end{minipage}
}
\\

\end{tabular}
\end{table*}

This section analyzes how different design choices affect model performance, including the reconstruction horizon, spatiotemporal encoder depth, source and chemical mapping module depth, spatiotemporal encoder architecture, and auxiliary variable inputs.
\textbullet~\textbf{Reconstruction Horizon.}
Table~\hyperref[tab:ablations_choices]{\ref*{tab:ablations_choices}a} shows that the reconstruction horizon strongly affects performance. A short horizon can lead to relatively small local mass drift,, but it exhibits instability during long-range propagation, whereas an overly long horizon increases error accumulation and optimization difficulty.
\textbullet~\textbf{Spatiotemporal Encoder Depth.}
Table~\hyperref[tab:ablations_choices]{\ref*{tab:ablations_choices}b} shows that a moderate encoder depth is preferable. A depth of $4$ leads to underfitting, with a pH error of $6.73\times10^{-4}$, whereas a depth of $12$ achieves the best error of $5.84\times10^{-4}$. Further increasing the depth to $24$ does not improve accuracy; instead, it worsens both local and global drift, suggesting that an overly deep encoder may introduce training instability or excessive smoothing.
\textbullet~\textbf{Source and Chemistry Module Depth.}
Table~\hyperref[tab:ablations_choices]{\ref*{tab:ablations_choices}c} shows that the U-Net depth affects both accuracy and conservation. Although a depth of $2$ yields a low pH error of $5.75\times10^{-4}$, its local and global drift increase to $8.41$\textperthousand\, and $19.1$\textperthousand. A depth of $4$ provides a better trade-off, reducing the drift to $2.21$\textperthousand\, and $10.3$\textperthousand\, while maintaining low pH error.
\textbullet~\textbf{Spatiotemporal Encoder Architecture.}
Table~\hyperref[tab:ablations_choices]{\ref*{tab:ablations_choices}d} compares different encoder architectures. The 2D CNN and ViT achieve reasonably low pointwise errors, but exhibit very large global drift, indicating that the lack of explicit temporal modeling leads to severe instability in long-sequence reconstruction. The 3D CNN clearly fails on this task. ViViT achieves the best pH error of $5.84\times10^{-4}$ with lower drift, demonstrating the benefit of explicit spatiotemporal attention.
\textbullet~\textbf{Auxiliary Variable Contribution.}
Table~\hyperref[tab:ablations_choices]{\ref*{tab:ablations_choices}e} shows that \(p\mathrm{CO}_2^s\) is the most important auxiliary variable. Removing \(p\mathrm{CO}_2^s\) increases the pH error from $5.84\times10^{-4}$ to $2.17\times10^{-3}$, indicating its central role in providing carbonate system constraints. Other variables also contribute to reconstruction performance.

\section{Conclusion}
\label{sec:conclusion}

We presented REACT, an active-tracer reconstruction framework for global sea surface pH that models pH through an underlying carbonate-system state. By combining conservative transport, source correction, and carbonate-equilibrium constraints, REACT achieves the best pH NRMSE and chemical consistency on multiple reconstruction tasks.
Beyond pH, this framework can in principle extend to other marine active tracers. Due to data and computational limitations, we leave the extension to other active tracers to future work, together with the limitation-driven directions discussed in Appendix~\ref{app:limitations}.

\newpage

\bibliographystyle{unsrtnat}
\bibliography{main}


\newpage
\appendix

\section{Related Work}
\label{app:related_work}

\subsection{Oceanographic Data}

The two main sources of oceanographic data are real-world observational data and model-based simulation data.

\paragraph{Observation data.}
Observation data are derived from direct in situ measurements of the ocean and therefore provide the most physically grounded source of supervision for ocean biogeochemical variables. Representative observational data sources include the WOD~\citep{WOD} and the SOCAT~\citep{SOCAT}. WOD is a long-term global archive of high fidelity oceanographic profile observations from diverse observing systems and historical cruises. It contains a wide range of essential ocean variables, including temperature, salinity, and carbonate-system-related measurements such as sea surface pH, making it a valuable resource for characterizing ocean variability over several decades. SOCAT, provides high-quality observations of \(p\mathrm{CO}_2^s\) from ships, drifters, and other surface observing platforms. These datasets offer complementary observational constraints on ocean biogeochemistry and air–sea carbon exchange. However, despite their high measurement fidelity, such observations remain highly sparse and irregular in both space and time because their sampling distributions are constrained by ship tracks, regional observing priorities, platform availability, and temporal measurement frequency.

\paragraph{Simulation data.}
Simulation data are generated by numerical ocean and Earth system models that explicitly couple ocean circulation with marine biogeochemical and carbon-cycle processes. Representative model systems include MOM6-COBALTv2~\citep{adcroft2019gfdl,stock2020ocean}, NEMO-PISCES~\citep{aumont2015pisces}, and GFDL-ESM4~\citep{dunne2020gfdl,stock2020ocean}. MOM6 forms the ocean component of GFDL OM4 model, which is used in GFDL climate and Earth system simulations, while COBALTv2 represents marine biogeochemical processes and ecosystem dynamics within the GFDL Earth system modeling framework. NEMO-PISCES couples the NEMO ocean circulation model with PISCES-v2, a biogeochemical model designed for carbon-cycle and ecosystem studies representing lower trophic levels and the cycles of carbon and major nutrients. GFDL-ESM4 further integrates ocean physics, biogeochemistry, ecosystems, atmosphere, sea ice, and land processes within a coupled carbon and chemistry modeling system. These simulation products provide continuous, gridded, and spatially complete fields of pH, \(p\mathrm{CO}_2^s\), DIC, and associated physical and biogeochemical variables, making them suitable as dense supervision signals and controlled benchmarks for learning-based models. Nevertheless, simulation outputs are not direct observations and may be affected by structural model assumptions, parameterization choices, numerical resolution, and errors in external forcing.



\subsection{Assimilation and Inversion Methods}

Since observed sea surface pH records are sparse, irregular, and unevenly distributed, ocean pH reconstruction has traditionally relied on assimilation and inversion methods to infer unobserved states. These methods estimate continuous ocean fields either by incorporating available measurements into numerical ocean models or by solving constrained optimization problems with prescribed physical, statistical, or smoothness assumptions.

\textbf{Assimilation Methods.}
Data assimilation methods combine observations with numerical model states to obtain estimates that are more consistent with ocean dynamics and biogeochemical processes. Representative approaches include Optimal Interpolation (OI)~\citep{conkright2002world-OI}, which infers missing values through covariance-based statistical weighting of nearby observations; variational methods such as 3D-Var and 4D-Var~\citep{courtier1994strategy-4D-var}, which optimize model states or controls by reducing the mismatch between observations and simulations; and ensemble-based methods such as the Ensemble Kalman Filter (EnKF)~\citep{nerger2008improving-KalmanFilter}, which sequentially updates model states while representing uncertainty through an ensemble. In large-scale ocean studies, these ideas are reflected in systems such as the Biogeochemical Southern Ocean State Estimate (B-SOSE)~\citep{verdy2017data-B-SOSE}, the Regional Ocean Modeling System (ROMS)~\citep{shchepetkin2005regional-ROMS}, MOM6~\citep{MOM6}, and the NASA Ocean Biogeochemical Model (NOBM)~\citep{rousseaux2012climate-ensemble-NOBM}. By exploiting model dynamics and multivariate coupling, these systems can improve the reconstruction of sparsely observed ocean variables. However, they are still constrained by observation scarcity, model structural errors, uncertain error statistics, and the high computational cost of repeated model integration.

\textbf{Inversion Methods.}
Inversion-based methods provide another important route for ocean field reconstruction. Instead of updating a forward model sequentially, they usually cast reconstruction as an optimization problem, seeking ocean states that fit the available observations while satisfying predefined physical, statistical, or regularization constraints. Classical examples include Variational Inverse Models (VIM)~\citep{brasseur1991application-VIM, belgacem2021climatological-VIM}, which reconstruct oceanographic fields by jointly minimizing data mismatch and regularization terms, such as spatial smoothness or dynamical consistency penalties. Compared with full assimilation systems, inversion methods provide a more direct and flexible formulation for reconstructing sparsely observed ocean biogeochemical variables. Nevertheless, their effectiveness depends heavily on the design of the objective function, the choice of regularization, and the prior assumptions imposed on the solution. In regions with extremely sparse observations, they may still have difficulty recovering fine-scale variability or correcting biased spatial structures.

\subsection{Data-Driven Reconstruction Methods}

Due to the heavy computational burden and modeling assumptions of assimilation and inversion methods, data-driven reconstruction approaches have received increasing attention for estimating sparse ocean fields.

\textbf{Supervised Learning from Observable Proxies.}
A common class of data-driven methods learns a mapping from more accessible proxy variables to target ocean variables. These proxies may include sea surface temperature, SSH, Chl., wind fields, or reanalysis-derived physical fields. Such models are usually trained in a supervised manner using observational products, numerical simulations, or remote sensing datasets. Representative examples include Earthformer~\citep{gao2022earthformer}, a transformer architecture designed for spatiotemporal Earth system data. Tree-based ensemble models, such as gcForest~\citep{zhong2025reconstruction-gcForest}, have also been applied to reconstruct ocean fields from satellite and reanalysis inputs. These methods can capture nonlinear relationships between multivariate predictors and target fields, making them useful for large-scale reconstruction. However, because their predictions are mainly driven by learned statistical correlations, they do not explicitly enforce physical or chemical consistency. As a result, their performance may degrade in sparsely sampled regions, unseen dynamical regimes, or distribution shifts beyond the training data.

\textbf{Foundation Models for Ocean Fields.}
Recent progress in large-scale pretraining has further introduced foundation models as a general framework for Earth and ocean field reconstruction. Models such as Prithvi~\citep{prithvi} and AtmoRep~\citep{atmorep} are trained on broad collections of geobiochemical variables across different spatial resolutions and temporal scales, aiming to learn transferable representations for multiple downstream tasks. For ocean reconstruction, such models have the potential to support interpolation, gap filling, forecasting, and cross-variable inference from sparse or heterogeneous inputs. Their scalability and pretraining paradigm make them suitable for general-purpose reconstruction across different ocean variables and regions. However, this line of research remain underexplored because of their heavy demand for large and diverse training datasets, lack of design level guarantees for physical consistency, and poor generalization to sparse, or dynamical complex ocean regimes.

\textbf{General AI Methods for Reconstruction.}
Beyond ocean-specific reconstruction models, general-purpose AI methods have also been explored for recovering sparse geobiochemical fields. Image inpainting approaches, including GAN-based models such as CoModGAN~\citep{zhao2021large-CoModGAN} and EdgeConnect~\citep{nazeri2019edgeconnect}, transformer-based methods such as the Mask-Aware Transformer (MAT)~\citep{li2022mat}, and diffusion-based methods such as RePaint~\citep{repaint}, treat gridded ocean variables as images and infer missing regions from learned spatial structures. These methods can handle large or irregular gaps, but they usually lack explicit constraints required by scientific variables. Neural operators learn mappings between function spaces and can approximate solution operators for dynamical systems and PDEs. Typical examples include Fourier Neural Operator (FNO)~\citep{fno}, which models fields in spectral space, and MeshGraphNets~\citep{pfaff2020learning-MeshGraphNets}, which can operate on irregular spatial domains. Although such methods support flexible prediction across resolutions, they often depend on high-quality simulation data and careful training. Implicit Neural Representations (INRs) model fields as continuous coordinate-based functions, enabling arbitrary-resolution reconstruction from sparse observations~\citep{luo2024inr}. Overall, these approaches provide flexible tools for ocean field reconstruction, but their application in scientific settings is still limited by data scarcity, weak physical consistency, and the multi-scale nature of geobiochemical processes.

\subsection{Physics-Informed Models}

To address the physical inconsistency of purely data-driven methods, recent studies have developed physics-constrained and hybrid learning frameworks for ocean field reconstruction.

\textbf{Physics-constrained and Hybrid Models.}
Physics-constrained models aim to improve the reliability and plausibility of learning-based reconstruction by incorporating physical equations, numerical schemes, or domain knowledge into the model design or training objective. Physics-informed neural networks (PINNs)~\citep{raissi2017physics-PINN} introduce governing equations or partial differential equation constraints directly into the loss function. Hybrid methods such as CLOINet~\citep{cutolo2024cloinet-CLONET} combine OI with deep learning, allowing the model to exploit both statistical interpolation and regime-dependent neural reconstruction. Similarly, 4DVarNet~\citep{beauchamp20234dvarnet} reformulates variational data assimilation as an end-to-end trainable neural framework, preserving an assimilation-inspired structure while learning reconstruction operators from data. Other methods, such as OxyGenerator~\citep{lu2024oxygenerator}, incorporate graph-based representations and oceanographic domain knowledge for long-term ocean field reconstruction. NUTS~\citep{nuts} further reconstructs ocean nutrients by decoupling ocean flows based on homogenization theory and applying mass conservation regularization to improve physical consistency. These hybrid approaches can enhance generalization and produce more plausible reconstructions than unconstrained black-box models. However, they are often sensitive to the specific design of physical constraints, network architectures, and training strategies. Moreover, they require substantial computational resources when applied to high-resolution or long-term reconstruction. More importantly, many existing physics-constrained reconstruction methods mainly focus on transport-dominated variables. However, for passive tracers, physical transport constraints alone are insufficient. There is still a critical lack of comprehensive physical and biochemical considerations to ensure fully consistent reconstructions in complex marine environments.

\section{Theory for Carbon-State Reconstruction}
\label{app:react-theory}

This appendix formalizes the claim in Section~\ref{sec:react-theory}. The result has two parts. The positive part shows that REACT provides a certificate for carbonate closure and source-free DIC conservation when the chemical and physical losses vanish. The negative part shows that a passive pH transport model cannot provide the same certificate in general: conserving the mass of pH does not imply conservation of the DIC mass implied by carbonate chemistry.

\subsection{Notation and setup}

Let \(\Omega\) be a compact smooth two-dimensional Riemannian domain representing the ocean surface, equipped with surface measure \(dx\). The domain may be closed, as on the sphere, or may have a smooth boundary \(\partial\Omega\). When \(\partial\Omega\neq\varnothing\), \(n\) denotes the outward unit normal. Let
\[
\Omega_T=\Omega\times[t_0,T].
\]
For any scalar field \(r:\Omega_T\to\mathbb R\), define its spatial inventory by
\[
M[r](t)=\int_\Omega r(x,t)\,dx.
\]
We use the convention that all equalities between functions hold almost everywhere unless otherwise stated.

Let \(h(x,t)\) denote pH, \(d(x,t)\) denote DIC, and \(c(x,t)\) denote surface \(p\mathrm{CO}_2\). Let \(\Gamma_D\) and \(\Gamma_C\) be the carbonate-equilibrium maps
\[
\Gamma_D(h,c)=\text{DIC implied by pH }h\text{ and }p\mathrm{CO}_2\ c,
\]
and
\[
\Gamma_C(h,d)=p\mathrm{CO}_2\text{ implied by pH }h\text{ and DIC }d.
\]
The chemical residual of a reconstructed pH field \(\hat h\) is
\[
\mathcal L_{\mathrm{chem}}(\hat h)
=
\|\Gamma_D(\hat h,c)-d\|_{L^2(\mathcal I)}^2
+
\beta
\|\Gamma_C(\hat h,d)-c\|_{L^2(\mathcal I)}^2,
\qquad \beta>0,
\]
where \(\mathcal I\subseteq\Omega_T\) is the observed space--time set.

REACT evolves a latent carbonate state \(z\) by a source-free conservative transport equation
\begin{equation}
\label{eq:app-react-transport}
\partial_t z+\nabla\cdot(u z)-\nabla\cdot(K_\theta\nabla z)=0,
\end{equation}
with no-flux boundary condition
\[
(u z-K_\theta\nabla z)\cdot n=0
\qquad\text{on }\partial\Omega.
\]
If \(\partial\Omega=\varnothing\), this boundary condition is omitted. The transported state is then corrected by an active source module,
\[
\tilde z_t=z_t+s_t,
\qquad
s_t=S_\theta(z_t,u,c,\xi),
\]
and decoded into pH by
\[
\hat h_t=G_\theta(\tilde z_t,c_t,\xi_t).
\]
For compactness we suppress \(\xi\) when it is not needed. The represented DIC of a latent state \(w\) is
\begin{equation}
\label{eq:app-represented-dic}
\mathcal D_\theta(w,c)=\Gamma_D(G_\theta(w,c),c).
\end{equation}
The source-free physical loss is
\begin{equation}
\label{eq:app-phys-loss}
\mathcal L_{\mathrm{phys}}
=
\int_{t_0}^{T}
\left|
M[\mathcal D_\theta(z_t,c_t)]
-
M[\mathcal D_\theta(z_{t_0},c_{t_0})]
\right|^2dt.
\end{equation}

\subsection{Assumptions}

We use the following assumptions only to make the theorem precise.

\paragraph{A1. Carbonate consistency and invertibility.}
There exists an admissible interval \(\mathsf H\) of pH values, an admissible interval \(\mathsf C\) of \(p\mathrm{CO}_2\) values, and an admissible interval \(\mathsf D\) of DIC values such that
\[
\Gamma_D:\mathsf H\times\mathsf C\to\mathsf D,
\qquad
\Gamma_C:\mathsf H\times\mathsf D\to\mathsf C
\]
are well-defined. For each fixed \(c\in\mathsf C\), the map
\[
h\mapsto \Gamma_D(h,c)
\]
is one-to-one on \(\mathsf H\). Its inverse in the first argument is denoted by
\[
H(D,c),
\]
so that
\[
\Gamma_D(H(D,c),c)=D.
\]
We also assume carbonate consistency:
\[
\Gamma_C(H(D,c),D)=c.
\]

\paragraph{A2. Observation identifiability.}
The observed set \(\mathcal I\) is sufficient to identify the chemical residuals in the admissible reconstruction class. Concretely, if an admissible \(\hat h\) satisfies
\[
\Gamma_D(\hat h,c)=d
\quad\text{and}\quad
\Gamma_C(\hat h,d)=c
\qquad\text{on }\mathcal I,
\]
then the same identities hold on the target domain under consideration. In the simplest case, \(\mathcal I=\Omega_T\), and this assumption is automatic. Without this assumption, the conclusion below holds on \(\mathcal I\).

\paragraph{A3. Realizability.}
For the realizability part of the theorem, assume the true DIC field decomposes as
\[
d(x,t)=d^0(x,t)+r(x,t),
\]
where \(d^0\) is the source-free transported DIC branch and \(r\) is the active carbon residual. The source-free branch satisfies
\begin{equation}
\label{eq:app-true-dic-transport}
\partial_t d^0+\nabla\cdot(u d^0)-\nabla\cdot(K^\star\nabla d^0)=0,
\end{equation}
with no-flux boundary condition
\[
(u d^0-K^\star\nabla d^0)\cdot n=0.
\]
Assume the model class contains the DIC-coordinate realization
\[
z_t=d^0_t,
\qquad
K_\theta=K^\star,
\qquad
S_\theta(z_t,u,c,\xi)=r_t,
\qquad
G_\theta(w,c)=H(w,c).
\]
This assumption does not require every implementation to use raw DIC as \(z\); it only says that the model class is expressive enough to represent such a carbon-coordinate solution.

\subsection{Positive result: carbon-state certificate}

\begin{theorem}[Carbon-state certificate and realizability]
\label{thm:react-positive}
Suppose A1--A2 hold and \(\beta>0\). Let \(z_t\), \(\tilde z_t\), and \(\hat h_t\) be a REACT trajectory, with represented DIC \(\mathcal D_\theta\) defined by \eqref{eq:app-represented-dic}.

\begin{itemize}
    \item If \(\mathcal L_{\mathrm{chem}}=0\), then the reconstructed pH closes carbonate chemistry:
    \[
    \Gamma_D(\hat h,c)=d
    \qquad\text{and}\qquad
    \Gamma_C(\hat h,d)=c
    \]
    on the target domain identified by \(\mathcal I\). In particular, if \(\mathcal I=\Omega_T\), these equalities hold almost everywhere on \(\Omega_T\).

    \item If \(\mathcal L_{\mathrm{phys}}=0\), then the source-free branch conserves represented DIC mass:
    \[
    M[\mathcal D_\theta(z_t,c_t)]
    =
    M[\mathcal D_\theta(z_{t_0},c_{t_0})]
    \]
    for almost every \(t\in[t_0,T]\). If the map
    \[
    t\mapsto M[\mathcal D_\theta(z_t,c_t)]
    \]
    is continuous, then the equality holds for every \(t\in[t_0,T]\).

    \item If, in addition, A3 holds and the true fields are chemically consistent,
    \[
    d=\Gamma_D(h,c),
    \qquad
    c=\Gamma_C(h,d),
    \]
    then there exists a REACT realization with
    \[
    \hat h=h,
    \qquad
    \mathcal L_{\mathrm{chem}}=0,
    \qquad
    \mathcal L_{\mathrm{phys}}=0.
    \]
    If the pH observations are noiseless, this same realization also gives
    \[
    \mathcal L_{\mathrm{pH}}=0.
    \]
\end{itemize}
\end{theorem}

\begin{proof}
We prove the three claims separately.

\paragraph{1. Chemical closure.}
By definition,
\[
\mathcal L_{\mathrm{chem}}
=
\|\Gamma_D(\hat h,c)-d\|_{L^2(\mathcal I)}^2
+
\beta
\|\Gamma_C(\hat h,d)-c\|_{L^2(\mathcal I)}^2.
\]
Both terms are nonnegative, and \(\beta>0\). Therefore \(\mathcal L_{\mathrm{chem}}=0\) implies
\[
\|\Gamma_D(\hat h,c)-d\|_{L^2(\mathcal I)}^2=0
\]
and
\[
\|\Gamma_C(\hat h,d)-c\|_{L^2(\mathcal I)}^2=0.
\]
Hence
\[
\Gamma_D(\hat h,c)=d
\qquad\text{and}\qquad
\Gamma_C(\hat h,d)=c
\]
almost everywhere on \(\mathcal I\). By the observation identifiability assumption A2, these identities hold on the target domain. If \(\mathcal I=\Omega_T\), this conclusion is immediate without any additional identifiability argument.

\paragraph{2. Source-free represented-DIC conservation.}
Define
\[
m(t)
=
M[\mathcal D_\theta(z_t,c_t)]
-
M[\mathcal D_\theta(z_{t_0},c_{t_0})].
\]
Then
\[
\mathcal L_{\mathrm{phys}}
=
\int_{t_0}^{T}|m(t)|^2dt.
\]
Since \(|m(t)|^2\ge 0\), the equality \(\mathcal L_{\mathrm{phys}}=0\) implies \(m(t)=0\) for almost every \(t\). Thus
\[
M[\mathcal D_\theta(z_t,c_t)]
=
M[\mathcal D_\theta(z_{t_0},c_{t_0})]
\]
for almost every \(t\). If \(m\) is continuous in time, then an almost-everywhere zero continuous function is identically zero, so the equality holds for every \(t\in[t_0,T]\).

\paragraph{3. Realizability.}
Assume A3. Choose the DIC-coordinate realization
\[
z_t=d^0_t,
\qquad
K_\theta=K^\star,
\qquad
S_\theta(z_t,u,c,\xi)=r_t,
\qquad
G_\theta(w,c)=H(w,c).
\]
Then the corrected latent state is
\[
\tilde z_t=z_t+s_t=d^0_t+r_t=d_t.
\]
The decoded pH is therefore
\[
\hat h_t
=
G_\theta(\tilde z_t,c_t)
=
H(d_t,c_t).
\]
Because the true fields are chemically consistent and \(H\) is the inverse of \(h\mapsto\Gamma_D(h,c)\), we have
\[
H(d_t,c_t)=h_t.
\]
Hence \(\hat h=h\). It follows immediately that \(\mathcal L_{\mathrm{pH}}=0\) for noiseless pH observations.

Next, since \(\hat h=h\) and the true fields are chemically consistent,
\[
\Gamma_D(\hat h,c)=\Gamma_D(h,c)=d
\]
and
\[
\Gamma_C(\hat h,d)=\Gamma_C(h,d)=c.
\]
Therefore \(\mathcal L_{\mathrm{chem}}=0\).

Finally, for the source-free branch,
\[
\mathcal D_\theta(z_t,c_t)
=
\Gamma_D(G_\theta(z_t,c_t),c_t)
=
\Gamma_D(H(z_t,c_t),c_t)
=
z_t
=
d^0_t.
\]
Thus the represented DIC of the source-free branch is exactly \(d^0_t\). Integrating the source-free DIC equation \eqref{eq:app-true-dic-transport} over \(\Omega\) gives
\[
\frac{d}{dt}M[d^0](t)
=
-\int_{\partial\Omega}
(u d^0-K^\star\nabla d^0)\cdot n\,dS.
\]
The no-flux boundary condition makes the boundary integral vanish. Therefore
\[
M[d^0](t)=M[d^0](t_0)
\qquad
\text{for all }t\in[t_0,T].
\]
Because \(\mathcal D_\theta(z_t,c_t)=d^0_t\), we obtain
\[
M[\mathcal D_\theta(z_t,c_t)]
=
M[\mathcal D_\theta(z_{t_0},c_{t_0})],
\]
and hence \(\mathcal L_{\mathrm{phys}}=0\). This proves the theorem.
\end{proof}

\subsection{Negative result: passive pH transport is not a DIC-conservation certificate}
\label{app:passive-counterexample}

We now show that the passive pH baseline lacks the certificate above. The issue is not that passive pH transport cannot conserve pH mass. It can. The issue is that pH mass is not the conserved carbonate inventory. The conserved source-free inventory is DIC, and the map from pH to DIC is nonlinear.

First, we record a purely algebraic fact.

\begin{lemma}[pH mass does not determine DIC mass]
\label{lem:mass-nonlinear}
Let \(\gamma:\mathsf H\to\mathbb R\) be a non-affine continuous function, where \(\mathsf H\subset\mathbb R\) is an interval. Then there exist two bounded pH fields \(h_1,h_2:\Omega\to\mathsf H\) such that
\[
M[h_1]=M[h_2],
\]
but
\[
M[\gamma(h_1)]\neq M[\gamma(h_2)].
\]
In particular, when \(\gamma(h)=\Gamma_D(h,c_0)\) for fixed \(c_0\), conservation of pH mass alone does not imply conservation of implied DIC mass.
\end{lemma}

\begin{proof}
Because \(\gamma\) is non-affine, there exist \(a,b\in\mathsf H\) and \(\lambda\in(0,1)\) such that, with
\[
m=\lambda a+(1-\lambda)b,
\]
we have
\[
\gamma(m)\neq \lambda\gamma(a)+(1-\lambda)\gamma(b).
\]
Choose a measurable subset \(A\subset\Omega\) with \(|A|=\lambda|\Omega|\). Define
\[
h_1(x)=m
\]
and
\[
h_2(x)=
\begin{cases}
a, & x\in A,\\
b, & x\in \Omega\setminus A.
\end{cases}
\]
Then
\[
M[h_1]
=
|\Omega|m
=
|\Omega|\{\lambda a+(1-\lambda)b\}
=
M[h_2].
\]
However,
\[
M[\gamma(h_1)]
=
|\Omega|\gamma(m),
\]
whereas
\[
M[\gamma(h_2)]
=
|\Omega|\{\lambda\gamma(a)+(1-\lambda)\gamma(b)\}.
\]
These two quantities are different by construction. Therefore equal pH mass does not determine equal implied DIC mass.
\end{proof}

The next theorem gives a dynamical counterexample for the passive advection--diffusion baseline.

\begin{theorem}[Passive pH transport can conserve pH mass while changing implied DIC mass]
\label{thm:passive-negative}
Fix a constant \(p\mathrm{CO}_2\) value \(c_0\), and define
\[
\gamma(h)=\Gamma_D(h,c_0).
\]
Assume \(\gamma\in C^2(\mathsf H)\) and \(\gamma\) is nonlinear in the sense that there exists \(a\in\mathsf H\) with
\[
\gamma''(a)\neq 0.
\]
Let \(h\) solve the passive pH transport equation
\begin{equation}
\label{eq:app-passive-ph}
\partial_t h+\nabla\cdot(u h)-\nabla\cdot(K\nabla h)=0
\end{equation}
on \(\Omega\), with no-flux boundary condition
\[
u\cdot n=0,
\qquad
K\nabla h\cdot n=0
\qquad
\text{on }\partial\Omega.
\]
Assume \(u\) is divergence-free,
\[
\nabla\cdot u=0,
\]
and \(K(x)\) is symmetric positive definite on some open set \(U\subseteq\Omega\).

Then there exists a smooth nonconstant initial pH field \(h_0:\Omega\to\mathsf H\) such that the corresponding solution satisfies
\[
M[h](t)=M[h](t_0)
\qquad
\text{for all }t,
\]
but
\[
\left.
\frac{d}{dt}
M[\Gamma_D(h,c_0)](t)
\right|_{t=t_0}
\neq 0.
\]
Consequently,
\[
M[\Gamma_D(h,c_0)](t)
\]
is not conserved, even though passive pH mass is conserved.
\end{theorem}

\begin{proof}
First, integrate \eqref{eq:app-passive-ph} over \(\Omega\):
\[
\frac{d}{dt}M[h](t)
=
\int_\Omega \partial_t h\,dx
=
-\int_\Omega \nabla\cdot(u h)\,dx
+
\int_\Omega \nabla\cdot(K\nabla h)\,dx.
\]
By the divergence theorem,
\[
\frac{d}{dt}M[h](t)
=
-\int_{\partial\Omega}u h\cdot n\,dS
+
\int_{\partial\Omega}K\nabla h\cdot n\,dS.
\]
The no-flux boundary conditions imply both boundary terms vanish. Hence
\[
\frac{d}{dt}M[h](t)=0,
\]
so
\[
M[h](t)=M[h](t_0)
\]
for all \(t\).

Now consider the DIC implied by pH and fixed \(p\mathrm{CO}_2\):
\[
D_{\mathrm{pass}}(x,t)=\Gamma_D(h(x,t),c_0)=\gamma(h(x,t)).
\]
Its mass derivative is
\[
\frac{d}{dt}M[D_{\mathrm{pass}}](t)
=
\frac{d}{dt}\int_\Omega \gamma(h(x,t))\,dx
=
\int_\Omega \gamma'(h)\partial_t h\,dx.
\]
Using \eqref{eq:app-passive-ph},
\[
\frac{d}{dt}M[D_{\mathrm{pass}}](t)
=
-\int_\Omega \gamma'(h)\nabla\cdot(u h)\,dx
+
\int_\Omega \gamma'(h)\nabla\cdot(K\nabla h)\,dx.
\]
Because \(\nabla\cdot u=0\), we have \(\nabla\cdot(u h)=u\cdot\nabla h\). Therefore
\[
-\int_\Omega \gamma'(h)\nabla\cdot(u h)\,dx
=
-\int_\Omega \gamma'(h)u\cdot\nabla h\,dx
=
-\int_\Omega u\cdot\nabla \gamma(h)\,dx.
\]
By the divergence theorem,
\[
-\int_\Omega u\cdot\nabla \gamma(h)\,dx
=
-\int_{\partial\Omega}\gamma(h)u\cdot n\,dS
+
\int_\Omega \gamma(h)\nabla\cdot u\,dx
=
0.
\]
For the diffusion term, integration by parts gives
\[
\int_\Omega \gamma'(h)\nabla\cdot(K\nabla h)\,dx
=
\int_{\partial\Omega}\gamma'(h)K\nabla h\cdot n\,dS
-
\int_\Omega \gamma''(h)\nabla h^\top K\nabla h\,dx.
\]
The boundary term vanishes by \(K\nabla h\cdot n=0\). Hence
\begin{equation}
\label{eq:app-dic-derivative}
\frac{d}{dt}M[D_{\mathrm{pass}}](t)
=
-\int_\Omega
\gamma''(h)
\nabla h^\top K\nabla h
\,dx.
\end{equation}

It remains to choose an initial field for which the right-hand side is nonzero at \(t=t_0\). Since \(\gamma''(a)\neq0\), by continuity there exists a small interval \(J\subset\mathsf H\) containing \(a\) such that \(\gamma''\) has a fixed nonzero sign on \(J\). Choose a smooth nonconstant function \(\phi\) supported inside \(U\), and choose \(\varepsilon\neq0\) sufficiently small that
\[
h_0(x)=a+\varepsilon\phi(x)
\]
takes values in \(J\). Then
\[
\nabla h_0=\varepsilon\nabla\phi.
\]
Evaluating \eqref{eq:app-dic-derivative} at \(t=t_0\) yields
\[
\left.
\frac{d}{dt}M[D_{\mathrm{pass}}](t)
\right|_{t=t_0}
=
-\varepsilon^2
\int_\Omega
\gamma''(a+\varepsilon\phi)
\nabla\phi^\top K\nabla\phi
\,dx.
\]
On the support of \(\phi\), \(K\) is positive definite, and \(\nabla\phi\) is nonzero on a set of positive measure because \(\phi\) is nonconstant. Also, \(\gamma''(a+\varepsilon\phi)\) has a fixed nonzero sign. Therefore the integral is nonzero, and so
\[
\left.
\frac{d}{dt}M[D_{\mathrm{pass}}](t)
\right|_{t=t_0}
\neq0.
\]
Thus the implied DIC mass changes at \(t_0\), even though \(M[h]\) is conserved. This proves the claim.
\end{proof}

\paragraph{Remark.}
The diffusion term is the essential mixing mechanism in Theorem~\ref{thm:passive-negative}. In the degenerate case of purely incompressible advection with no diffusion, the full distribution of \(h\) is transported, so \(\int_\Omega \gamma(h)\,dx\) is also conserved for any fixed nonlinear \(\gamma\). That special case does not provide a certificate for passive pH reconstruction in the advection--diffusion setting used by the baseline. Once mixing or coarse-grained diffusion is present, nonlinear carbonate chemistry makes conservation of \(M[h]\) insufficient for conservation of \(M[\Gamma_D(h,c)]\).

\subsection{Consequence for REACT versus passive pH transport}

Theorem~\ref{thm:react-positive} and Theorem~\ref{thm:passive-negative} formalize the modeling distinction. REACT transports a carbon-related state and evaluates the source-free branch through represented DIC,
\[
\mathcal D_\theta(z_t,c_t)=\Gamma_D(G_\theta(z_t,c_t),c_t).
\]
Thus zero physical loss directly certifies conservation of the represented DIC inventory. In contrast, the passive baseline transports pH itself. Even when that baseline exactly conserves
\[
M[h](t),
\]
Theorem~\ref{thm:passive-negative} shows that it can change
\[
M[\Gamma_D(h,c)](t),
\]
which is the chemically implied DIC inventory. Therefore passive pH transport has the right conservative PDE form but applies it to the wrong variable. REACT avoids this mismatch by placing the source-free transport constraint on a carbon coordinate and decoding pH only after conservative carbon transport and active carbon-cycle correction.

\section{Implementation Details of REACT}\label{app::REACT-structure}

\subsection{Architectural Details }

REACT consists of an initial state reconstruction module, a conservative transport module, a source module, a pH decoder and a chemical module.

\textbf{Initial State Reconstruction Module.}
The initial state reconstruction module estimates the initial carbon-related state from sparse observations and auxiliary fields. The input contains pH observations, \(p\mathrm{CO}_2^s\) observations, and nine MOM6-derived variables. We use a ViViT-style~\citep{vivit} spatiotemporal transformer as the backbone of this module. The input sequence is divided into non-overlapping spatiotemporal tubelets, and each tubelet is projected into a latent token. The tokens are then processed by stacked transformer blocks to capture long-range spatial dependencies and temporal correlations. The output tokens are reshaped back to the global latitude--longitude grid to produce the initial carbon-related transport state:
\[
    z_0 = F_{\theta}(\mathcal O_h, \mathcal O_c, \mathbf X_{\mathrm{MOM6}}),
\]
where $\mathcal O_h$ denotes sparse pH observations, $\mathcal O_c$ denotes sparse \(p\mathrm{CO}_2^s\) observations, and $\mathbf X_{\mathrm{MOM6}}$ denotes the nine MOM6-derived variables. The reconstructed state $z_0$ serves as the initial carbon-related transport state for subsequent physical propagation.

\textbf{Conservative Transport Module.}
The evolution of the reconstructed carbon-related state is governed by the standard advection-diffusion equation:
\[
    \frac{\partial z}{\partial t}
    =
    -\nabla \cdot (z\mathbf w)
    +
    \nabla \cdot (K\nabla z),
\]
where $z$ is the carbon-related transport state, $\mathbf w=(u,v)$ denotes the horizontal surface velocity field, and $K$ is the diffusion coefficient. The advection term describes flow-driven transport, while the diffusion term describes unresolved spreading. In our implementation, the diffusion term is removed and only the advective transport term is retained. Given the initial state $z_0$, the module produces the propagated states
\[
    z_t = \mathcal P(z_0,\mathbf w_{1:t}), \quad t=1,\ldots,T,
\]
where $\mathcal P$ denotes the physical transport operator.

\textbf{Source Module.}
The source module parameterizes the non-conservative residual in the carbon-state evolution. After the conservative transport module produces $z_t$, the source module predicts a source correction
\[
    s_t = S_{\theta}(z_t),
\]
where $S_{\theta}$ is implemented as a U-Net~\citep{ronneberger2015unet}. The corrected carbon-related state is then obtained by
\[
    \tilde z_t = z_t + s_t.
\]
Here $s_t$ represents unresolved source and sink effects associated with air--sea CO$_2$ exchange, biological activity, and other biogeochemical processes that cannot be captured by source-free transport alone.

\textbf{pH Decoder.}
The pH decoder maps the corrected carbon-related state to the reconstructed sea surface pH field. Given the corrected state $\tilde z_t$ and the carbonate equilibrium constants $K_1$ and $K_2$, the decoder predicts
\[
    \hat h_t = G_{\theta}(\tilde z_t,K_{1,t},K_{2,t}),
\]
where $G_{\theta}$ is implemented as a U-Net. The constants $K_1$ and $K_2$ denote the first and second dissociation constants of the carbonate system, which characterize the nonlinear equilibrium among carbonate species and hydrogen ions. By conditioning on $K_1$ and $K_2$, the decoder maps the corrected carbon-related state to pH under carbonate-equilibrium conditions.

\textbf{Chemical Module.}
The chemical module implements carbonate-equilibrium formulas that couple reconstructed pH, \(p\mathrm{CO}_2^s\), and DIC. Given reconstructed pH, it computes the DIC implied by observed \(p\mathrm{CO}_2^s\) and the \(p\mathrm{CO}_2^s\) implied by observed DIC. These calculations are used to evaluate and enforce carbonate-system consistency. Standard carbonate-system calculations can be performed by CO2SYS~\citep{sharp2020co2sysv3}~(under MIT license)\footnote{https://github.com/jamesorr/CO2SYS-MATLAB} or PyCO2SYS~\citep{humphreys2022pyco2sys}~(under GPL-3.0 license)\footnote{https://github.com/mvdh7/PyCO2SYS}. However, these solvers are not directly used as differentiable PyTorch layers in our end-to-end training pipeline. We therefore implement the carbonate-equilibrium calculation with explicit analytical formulas, following the steps below:

Given the reconstructed pH $\hat h_t$, the hydrogen ion concentration is
\[
    [\mathrm H^+]_t = 10^{-\hat h_t}.
\]
The dissolved CO$_2$ concentration is computed from \(p\mathrm{CO}_2^s\) as
\[
    [\mathrm{CO}_2^\ast]_t = K_{0,t}\mathcal O_{c,t},
\]
where $K_{0,t}$ is the CO$_2$ solubility coefficient and $\mathcal O_{c,t}$ denotes \(p\mathrm{CO}_2^s\) observations. The bicarbonate and carbonate ion concentrations are then given by
\[
    [\mathrm{HCO}_3^-]_t
    =
    \frac{K_{1,t}[\mathrm{CO}_2^\ast]_t}{[\mathrm H^+]_t},
    \qquad
    [\mathrm{CO}_3^{2-}]_t
    =
    \frac{K_{1,t}K_{2,t}[\mathrm{CO}_2^\ast]_t}{[\mathrm H^+]_t^2},
\]
where $K_{1,t}$ and $K_{2,t}$ are the first and second dissociation constants of the carbonate system. The DIC implied by the reconstructed pH is
\[
    \hat d_t
    =
    [\mathrm{CO}_2^\ast]_t
    +
    [\mathrm{HCO}_3^-]_t
    +
    [\mathrm{CO}_3^{2-}]_t ,
\]
or equivalently,
\[
    \hat d_t
    =
    K_{0,t}\mathcal O_{c,t}
    \left(
    1
    +
    \frac{K_{1,t}}{[\mathrm H^+]_t}
    +
    \frac{K_{1,t}K_{2,t}}{[\mathrm H^+]_t^2}
    \right).
\]

Conversely, when DIC observations are available, the \(p\mathrm{CO}_2^s\) implied by the reconstructed pH can be obtained by inverting the same carbonate-equilibrium relation. Specifically, given DIC observation $\mathcal O_{d,t}$, we compute
\[
    \hat c_t
    =
    \frac{
        \mathcal O_{d,t}
    }{
        K_{0,t}
        \left(
        1
        +
        \frac{K_{1,t}}{[\mathrm H^+]_t}
        +
        \frac{K_{1,t}K_{2,t}}{[\mathrm H^+]_t^2}
        \right)
    } .
\]
Here $\hat c_t$ denotes the \(p\mathrm{CO}_2^s\) implied by reconstructed pH and DIC observation $\mathcal O_{d,t}$.

To verify the consistency of the formula-based implementation with PyCO2SYS, we compute DIC using both methods, with pH and sea surface $\text{CO}_2$ partial pressure from MOM6 as inputs. The mean relative difference between the formula-based DIC and the PyCO2SYS DIC is $3.736\times 10^{-3}$. The small discrepancy indicates that the formula-based implementation is an appropriate substitution for PyCO2SYS, while remaining compatible with neural-network training.

\subsection{Input and Output Details}
\label{app:input_output_details}

REACT accepts an input tensor of shape $B \times C \times T_{\mathrm{in}} \times H \times W$, where $B$ is the batch size, $T_{\mathrm{in}}$ is the number of input temporal frames, $H \times W$ is the spatial resolution, and $C=11$ denotes the input channels. These channels consist of pH observations, \(p\mathrm{CO}_2^s\) observations, and nine MOM6-derived variables. The model outputs a tensor of shape $B \times 1 \times T_{\mathrm{out}} \times H \times W$, representing the reconstructed sea surface pH sequence. For monthly-to-weekly reconstruction, $T_{\mathrm{out}}=4T_{\mathrm{in}}$, while for weekly-to-daily reconstruction, $T_{\mathrm{out}}=7T_{\mathrm{in}}$.

The initial state reconstruction module processes the full spatiotemporal input sequence with tubelet embedding. The input tensor is partitioned into $N = T_{\mathrm{in}}HW/(tpq)$ non-overlapping tubelets, where $t$, $p$, and $q$ denote the temporal and spatial tubelet sizes. Each tubelet is flattened into a vector of size $Ctpq$ and projected into a latent space of dimension $D$. The ViViT-style encoder maps these tokens back to the latitude--longitude grid and produces the initial carbon-related transport state $z_0 \in \mathbb R^{B \times 1 \times H \times W}$.

Starting from $z_0$, the conservative transport module produces a temporally resolved carbon-related state sequence $\{z_t\}_{t=1}^{T_{\mathrm{out}}} \in \mathbb R^{B \times 1 \times T_{\mathrm{out}} \times H \times W}$. The source module then parameterizes unresolved non-conservative effects as frame-wise residual corrections to the transported state, yielding the corrected carbon-related state sequence $\{\tilde z_t\}_{t=1}^{T_{\mathrm{out}}} \in \mathbb R^{B \times 1 \times T_{\mathrm{out}} \times H \times W}$.

The pH decoder maps the corrected carbon-related state sequence to sea surface pH under carbonate equilibrium conditions. The equilibrium constants $K_1$ and $K_2$ provide the nonlinear chemical link between pH and the carbonate system. The decoder outputs the reconstructed pH sequence $\hat{\mathbf H}=\{\hat h_t\}_{t=1}^{T_{\mathrm{out}}} \in \mathbb R^{B \times 1 \times T_{\mathrm{out}} \times H \times W}$. The chemical module is used when carbonate-system variables are required for chemical constraints or evaluation, while the primary output of REACT remains the reconstructed sea surface pH sequence.

\section{Experiment Details}
\label{app:experiment_details}

\subsection{Datasets}
\label{app:dataset}

\textbf{WOD~\citep{WOD}\footnote{https://www.ncei.noaa.gov/access/world-ocean-database-select/dbsearch.html}.}
WOD is a publicly available global archive of quality-controlled ocean profile data maintained by NOAA NCEI. The WOD 2023 release contains scientifically quality-controlled ocean profile and plankton data collected over the global ocean from 1772 to 2022, and includes measurements of temperature, salinity, alkalinity, pH, nutrients, and other oceanographic variables. In this work, we use surface pH records from WOD 2023 to build sparse pH observations for sea surface pH reconstruction. The selected records cover the period from January 1959 to December 2022.

\textbf{SOCAT~\citep{SOCAT}~(under CC BY 4.0 license)\footnote{https://socat.info/index.php/version-2023/}.}
SOCAT is a community synthesis of quality-controlled surface ocean CO$_2$ observations, with the released data primarily reported as fugacity of CO$_2$ ($f\mathrm{CO}_2$). These measurements are collected from ships, moorings, autonomous platforms, and drifting surface platforms over the global ocean and coastal seas. We use surface CO$_2$ records from SOCAT version 2023 as sparse carbon-system observations. The selected records cover the period from January 1959 to December 2022.

\textbf{MOM6-COBALTv2~\citep{MOM6}~(under LGPLv3 license)\footnote{https://github.com/NOAA-GFDL/MOM6-examples}.}
MOM6 is the sixth-generation Modular Ocean Model developed at NOAA GFDL for simulating ocean circulation, and COBALTv2 is the second-generation Carbon, Ocean Biogeochemistry, and Lower Trophics biogeochemical model used in GFDL Earth system modeling. We use MOM6-COBALTv2 outputs from January 1959 to December 2022 as dense gridded reference data. MOM6-COBALTv2 provides dense gridded ocean fields, including sea surface pH, surface CO$_2$, DIC, temperature, salinity, and velocity fields. All variables are regridded to a global $1^\circ \times 1^\circ$ grid. The daily fields are used to construct datasets for both monthly-to-weekly and weekly-to-daily reconstruction tasks.

\subsection{Metrics}
\label{app:metrics}

\textbf{Normalized RMSE.}
We use normalized root mean square error (NRMSE) to measure the discrepancy between the reconstructed field and the ground-truth field. We first compute the latitude-weighted RMSE as
{\small
\begin{align*}
    \mathrm{RMSE}
    =
    \frac{1}{N}
    \sum_{t=1}^{N}
    \sqrt{
    \frac{1}{HW}
    \sum_{h=1}^{H}
    \sum_{w=1}^{W}
    \alpha_h
    \left(
    y_{thw} - u_{thw}
    \right)^2
    },
\end{align*}
}
where $N$ is the number of time points, $H$ and $W$ denote the numbers of latitude and longitude grid points, respectively, $y_{thw}$ is the ground-truth value at time $t$ and grid location $(h,w)$, and $u_{thw}$ is the reconstructed value at time $t$ and grid location $(h,w)$. The latitude weight $\alpha_h$ accounts for the area variation across latitudes and is defined as
\[
    \alpha_h
    =
    \frac{\cos \phi_h}
    {
    \frac{1}{H}\sum_{h'=1}^{H}\cos \phi_{h'}
    },
\]
where $\phi_h$ is the latitude of the $h$-th grid row. The NRMSE is then obtained by normalizing RMSE with the mean of the ground-truth values:
{\small
\begin{align*}
    \mathrm{NRMSE}
    =
    \frac{\mathrm{RMSE}}
    {
    \frac{1}{NHW}
    \sum_{t=1}^{N}
    \sum_{h=1}^{H}
    \sum_{w=1}^{W}
    y_{thw}
    }.
\end{align*}
}

\textbf{pH.}
The reported pH NRMSE is computed between the reconstructed and ground-truth sea surface pH fields, given by
{\small
\begin{align*}
    \mathrm{NRMSE}_{\mathrm{pH}}
    =
    \mathrm{NRMSE}(\hat h, h).
\end{align*}
}
Here, $h$ and $\hat h$ denote the ground-truth and reconstructed sea surface pH fields, respectively.

\textbf{Chem.}
Let $\hat h$ be the reconstructed pH field, and let $c$ and $d$ be the reference \(p\mathrm{CO}_2^s\) and DIC fields. We denote by $\Gamma_D(\hat h,c)$ the DIC implied by reconstructed pH and \(p\mathrm{CO}_2^s\), and by $\Gamma_C(\hat h,d)$ the \(p\mathrm{CO}_2^s\) implied by reconstructed pH and DIC. The chemical consistency error is defined as
\[
    \mathrm{Chem.}
    =
    \frac{1}{2}
    \mathrm{NRMSE}
    \left(
    \Gamma_C(\hat h,d),
    c
    \right)
    +
    \frac{1}{2}
    \mathrm{NRMSE}
    \left(
    \Gamma_D(\hat h,c),
    d
    \right).
\]

\textbf{Drift$_{\mathrm{local}}$.}
Local drift measures the DIC inventory variation implied by the source-free branch within each prediction clip. Let $\hat h_{\mathrm{flow}}(t)$ be the source-free pH prediction at time $t$, and let $c_t$ be the reference \(p\mathrm{CO}_2^s\) field at the same time. The corresponding source-free DIC field is
\[
    \hat d_{\mathrm{flow}}(t)
    =
    \Gamma_D(\hat h_{\mathrm{flow}}(t),c_t).
\]
Let $M[\hat d_{\mathrm{flow}}](t)$ denote the spatial integral of $\hat d_{\mathrm{flow}}(t)$ over the ocean domain. For a clip $i$ with time points $\{t_{i,1},\ldots,t_{i,L}\}$, we define
\[
    D_i^{\mathrm{loc}}
    =
    \frac{1}{L-1}
    \sum_{\ell=2}^{L}
    \frac{
    \left|
    M[\hat d_{\mathrm{flow}}](t_{i,\ell})
    -
    M[\hat d_{\mathrm{flow}}](t_{i,1})
    \right|
    }{
    \left|
    M[\hat d_{\mathrm{flow}}](t_{i,1})
    \right|
    }.
\]
The reported local drift is averaged over all clips:
\[
    \mathrm{Drift}_{\mathrm{local}}
    =
    \frac{1}{C}
    \sum_{i=1}^{C}
    D_i^{\mathrm{loc}},
\]
where $C$ is the number of clips.

\textbf{Drift$_{\mathrm{global}}$.}
Global drift measures the DIC inventory variation implied by the source-free branch over longer evaluation segments. For a segment $s$ with time points $\{t_{s,1},\ldots,t_{s,T_s}\}$, we define
\[
    D_s^{\mathrm{glob}}
    =
    \frac{1}{T_s-1}
    \sum_{j=2}^{T_s}
    \frac{
    \left|
    M[\hat d_{\mathrm{flow}}](t_{s,j})
    -
    M[\hat d_{\mathrm{flow}}](t_{s,1})
    \right|
    }{
    \left|
    M[\hat d_{\mathrm{flow}}](t_{s,1})
    \right|
    }.
\]
Using the head and tail evaluation segments, the reported global drift is
\[
    \mathrm{Drift}_{\mathrm{global}}
    =
    \frac{1}{2}
    \left(
    D_{\mathrm{head}}^{\mathrm{glob}}
    +
    D_{\mathrm{tail}}^{\mathrm{glob}}
    \right).
\]

\subsection{Implementation Details}
\label{app:implementation_details}

We train the model for 250 epochs using the AdamW optimizer~\citep{adamw} with momentum parameters $(0.9,0.999)$. Cosine annealing is used for learning-rate scheduling. We select the learning rate and weight decay by grid search over $\{2,1,0.8,0.5,0.3,0.1,0.08\}\times 10^{-3}$ and $\{8,5,2,1,0.5,0.1\}\times 10^{-3}$, respectively. The selected learning rate and weight decay are $1\times10^{-3}$ and $5\times10^{-4}$ for the MOM6 simulation experiments, and $1\times10^{-4}$ and $5\times10^{-4}$ for the WOD real-observation experiments. 
We tune the mass-conservation loss weight over $\{0.5,0.1,0.01,0.001\}$. For the simulation experiments, we use a mass-conservation loss weight of $0.01$ for monthly-to-weekly reconstruction and $0.1$ for weekly-to-daily reconstruction. For the WOD real-observation experiments, we use a mass-conservation loss weight of $0.01$.
GradNorm~\citep{chen2018gradnorm} is applied in all experiments to adaptively balance the active loss terms. We set the GradNorm asymmetry parameter to $1.5$, use a learning rate of $0.025$ for the loss weights, and update the adaptive weights every $60$ iterations. For the relative task-priority coefficients in GradNorm, we assign a larger coefficient to pH reconstruction and smaller coefficients to the carbonate-consistency terms, using $2$ for pH and $0.5$ for both \(p\mathrm{CO}_2^s\) and DIC-based chemical constraints.
We conduct all experiments on 8 NVIDIA 6000D GPUs with a total batch size of 8 for both datasets.

\subsection{Sampling Strategies}
\label{app:sampling_strategy}

\begin{wrapfigure}{r}{0.50\linewidth}
    \centering
    \captionsetup{font=footnotesize,skip=2pt}
    \vspace{-0.5em}
    \includegraphics[width=0.98\linewidth]{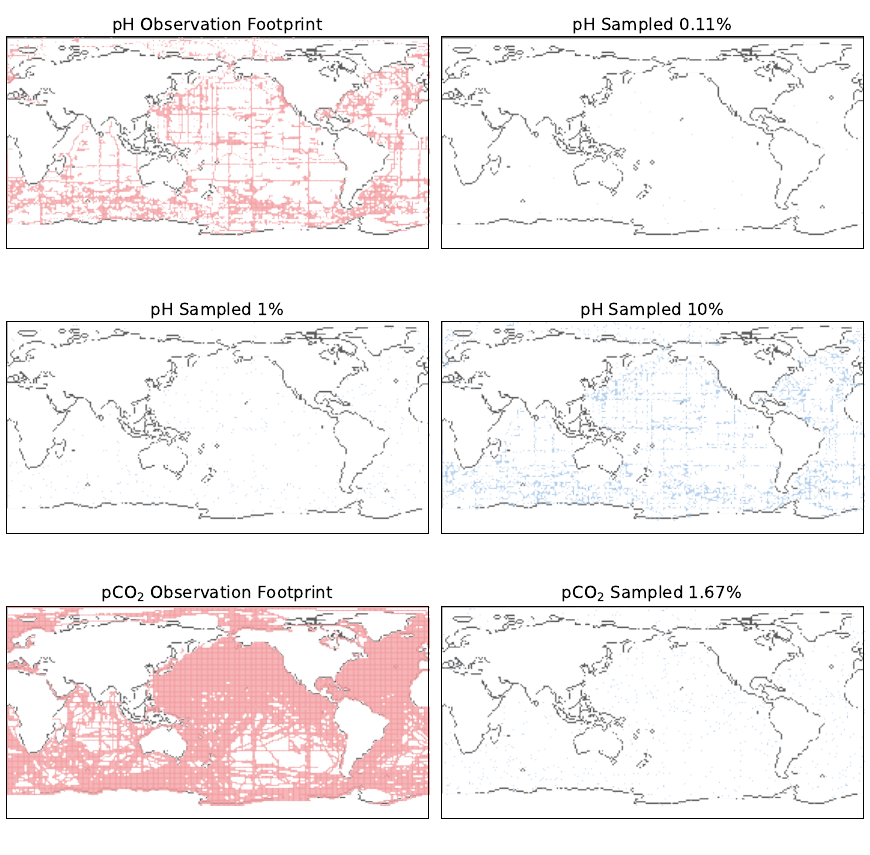}
    \caption{\textbf{Observation-guided sampling strategy.}
    Red points show historical observation footprints, defined as grid locations with at least one valid observation over the observational period. Blue points show sampled locations at one randomly selected time step under the target sampling ratios. pH is sampled at $0.11\%$, $1\%$, and $10\%$, while \(p\mathrm{CO}_2^s\) is sampled at $1.67\%$.}
    \label{fig:obs_guided_sampling}
    \vspace{-1.0em}
\end{wrapfigure}

\textbf{Simulation Dataset.}
To simulate realistic sparse observations, we use an observation-guided sampling strategy rather than uniformly sampling over the entire ocean domain, as shown in Figure~\ref{fig:obs_guided_sampling}. For each variable, we first collect all grid locations that have valid observations at least once over the observational period. This historical observation footprint defines the candidate sampling region.
Let $N_{\mathrm{ocean}}$ be the number of ocean grid points and let $r$ be the target sampling ratio. At each time step, we sample approximately $rN_{\mathrm{ocean}}$ ocean grid points, rounded to the nearest integer. Sampling is first performed uniformly from the historical observation footprint of the corresponding variable. If the footprint contains fewer points than required by the target ratio, all footprint points are selected and the remaining points are sampled from nearby ocean grid cells around the footprint until the target number is reached.
For pH, we use sampling ratios of $0.11\%$, $1\%$, and $10\%$. For \(p\mathrm{CO}_2^s\), we use a fixed sampling ratio of $1.67\%$, matching the sparsity level of SOCAT \(p\mathrm{CO}_2^s\) observations. This procedure preserves the spatial bias of real observational networks while allowing controlled comparisons under different sparsity levels. 

\textbf{Real Observations.}
For real-observation experiments, we first identify valid observation locations where pH and \(p\mathrm{CO}_2^s\) are both available. Since directly available DIC observations are not provided, we compute observational DIC using PyCO2SYS from WOD pH, SOCAT \(p\mathrm{CO}_2^s\), salinity, and temperature. The intersection of valid pH, \(p\mathrm{CO}_2^s\), and computed DIC locations is used as the available observation set.
We uniformly sample $75\%$ of the available observation locations as input observations and hold out the remaining $25\%$ as target observations. The input observations are used to construct the model input, while the held-out observations are used for evaluation.

\subsection{Hyperparameter setting.}
\label{app:model_params}

\textbf{REACT\footnote{https://anonymous.4open.science/r/REACT-67E7}.}
The model hyperparameters of REACT are summarized in Table~\ref{tab:react-params}. Training hyperparameters are provided in Appendix~\ref{app:implementation_details}.

\begin{table}[t]
\centering
\caption{Model hyperparameters used for REACT.}
\label{tab:react-params}
\setlength{\tabcolsep}{6pt}
\small
\begin{tabular}{p{3.0cm} p{8.2cm} >{\centering\arraybackslash}p{1.6cm}}
\toprule
\textbf{Hyper-param.} & \textbf{Meaning} & \textbf{Value} \\
\midrule
$t$ & Temporal tubelet size in the ViViT initializer & 1 \\
$p,q$ & Spatial patch size in the ViViT initializer & 12, 12 \\
$D$ & Embedding dimension in the ViViT initializer & 768 \\
\# Layers & Number of transformer layers in the ViViT initializer & 12 \\
\# Heads & Number of attention heads in each transformer layer & 12 \\
MLP ratio & Expansion ratio of the MLP layers & 4.0 \\
\# Source levels & Number of U-Net levels in the source module & 4 \\
\# Decoder levels & Number of U-Net levels in the pH decoder & 4 \\
Base channels & Number of channels in the first U-Net level & 64 \\
Dropout & Dropout rate in the ViViT initializer & 0.0 \\
Attention dropout & Attention dropout rate in the ViViT initializer & 0.0 \\
\bottomrule
\end{tabular}
\end{table}

\begin{table}[t]
\centering
\caption{Hyper-parameters used for all baseline models.}
\label{tab:baseline-params}
\setlength{\tabcolsep}{3pt}
\fontsize{7}{8}\selectfont
\begin{tabular}{p{1.8cm} p{2.4cm} p{7.2cm} >{\centering\arraybackslash}p{1.2cm}}
\toprule
\textbf{Baseline} & \textbf{Hyper-param.} & \textbf{Meaning} & \textbf{Value} \\
\midrule
NUTS
& $t$ & Tubelet size in the initializer & 2 \\
& $p,q$ & Patch size in the transformer layers & 18, 36 \\
& D & Embedding dimension & 768 \\
& \# Coarse layer & Number of transformer layers in the initializer & 12 \\
& \# Refine layer & Number of transformer layers in the refinement module & 6 \\
& \# Heads & Number of attention heads in each transformer layer & 16 \\
& MLP dimension & Hidden dimension of the MLP layers & 2048 \\
& \# Blocks & Number of convolutional blocks in the source/diffusion network & [2, 2] \\
& Base channels & Number of channels in the first convolutional block of source/diffusion network & 16 \\
& Dropout(coarse) & Dropout rate in data-driven initializer & 0.0 \\
& Dropout(refine) & Dropout rate in refinement module & 0.1 \\
\midrule
4D-VarNet
& Shape data & Input channels of the model & 11 \\
& D & Hidden dimension for bilinear units & 64 \\
& \# Blocks & Number of residual bilinear blocks & 5 \\
& Half kernel size & Convolution half-kernel sizes & 1 \\
& Subsample stride & Subsampling stride & 1 \\
& Dropout & Dropout rate applied in encoder & 0.0 \\
\midrule
Marble
& Grid size & Spatial resolution of modulation grid & 8 \\
& Grid base & Number of basis functions per grid cell & 32 \\
& Hidden dimension & Dimension of hidden layers in Siren~\citep{siren} network & 128 \\
& \# Layer & Number of layers in Siren~\citep{siren} network & 4 \\
& Latent dimension & Dimension of latent vector & 32 \\
& Modulation dimension & Hidden dimension of modulation network & 96 \\
\midrule
FNO
& \# Modes & Number of modes to keep in Fourier layer & 16 \\
& In Channels & Number of channels in input function & 11 \\
& Out Channels & Number of channels in output function & 1 \\
& Hidden Channels & Width of FNO & 64 \\
& Projection Channels & Number of channels in the projection block & 128 \\
& Lifting Channels & Number of channels in the lifting block & 64 \\
& \# Layers & Number of Fourier layers & 4 \\
\midrule
U-Net
& Kernel size & Kernel size of encoder/decoder block & 3 \\
& Stride & Stride of encoder/decoder block & 1 \\
& Padding size & Padding size of encoder/decoder block & 1 \\
& Base channels & Number of channels in the first encoder block & 64 \\
& \# Blocks & Number of encoder/decoder blocks & 4 \\
\midrule
ViT
& $p,q$ & Patch size & 12, 12 \\
& D & Embedding dimension & 768 \\
& \# Blocks & Number of ViT blocks & 12 \\
& \# Heads & Number of attention heads & 12 \\
& MLP dimension & The hidden dimension of the MLP layers & 3072 \\
& Dropout & Dropout rate & 0.0 \\
\midrule
AtmoRep
& D & Embedding Dimension & 2048 \\
& \# Encoder layer & Number of transformer layers in the encoder & 10 \\
& \# Encoder heads & Number of attention heads in the encoder & 16 \\
& \# Decoder layer & Number of transformer layers in the decoder & 10 \\
& \# Decoder heads & Number of attention heads in the decoder & 16 \\
& \# Tail nets & Number of networks in Tail Ensemble & 16 \\
& \# Tail layer & Number of layers per network in Tail Ensemble (0 represents linear) & 0 \\
\midrule
Prithvi
& $p,q$ & Patch size & 2, 2 \\
& D & Embedding Dimension & 2560 \\
& \# Encoder layer & Number of transformer layers in the encoder & 25 \\
& \# Decoder layer & Number of transformer layers in the decoder & 5 \\
& \# Heads & Number of attention heads in each transformer layer & 16 \\
& MLP dimension & Hidden dimension of the MLP layers & 10240 \\
& Window size & Window size of attention layers & 30, 32 \\
& Drop path & For stochastic depth~\citep{droppath} & 0.0 \\
& Dropout & Dropout rate & 0.0 \\
\bottomrule
\end{tabular}
\end{table}

\textbf{NUTS~\citep{nuts}~(under MIT license)\footnote{https://github.com/Leamonz/NUTS}.}
The hyperparameters used for NUTS are listed in Table~\ref{tab:baseline-params}. The model is trained for $250$ epochs using the AdamW optimizer with momentum parameters $(0.9,0.999)$. We set the learning rate to $5\times10^{-4}$ and the weight decay to $2\times10^{-3}$ for all experiments. Cosine annealing is used for learning-rate scheduling.

\textbf{4D-VarNet~\citep{beauchamp20234dvarnet} (under CeCILL-C license)\footnote{https://zenodo.org/records/8048541}.}
The hyperparameters used for 4D-VarNet are summarized in Table~\ref{tab:baseline-params}. The model is trained for $250$ epochs using the AdamW optimizer with momentum parameters $(0.9,0.999)$. We set the learning rate to $1\times10^{-4}$ and the weight decay to $1\times10^{-4}$ for all experiments. Cosine annealing is used for learning-rate scheduling.

\textbf{Marble~\citep{marble} (under MIT license)\footnote{https://github.com/LeapLabTHU/GridMix}.}
The hyperparameters used for Marble are summarized in Table~\ref{tab:baseline-params}. The model is trained for $250$ epochs using the AdamW optimizer with momentum parameters $(0.9,0.999)$. We set the learning rate to $1\times10^{-3}$ and the weight decay to $5\times10^{-3}$ for all experiments. Cosine annealing is used for learning-rate scheduling.

\textbf{FNO~\citep{fno} (under MIT license)\footnote{https://github.com/neuraloperator/neuraloperator}.}
The hyperparameters used for FNO are summarized in Table~\ref{tab:baseline-params}. The model is trained for $250$ epochs using the AdamW optimizer with momentum parameters $(0.9,0.999)$. We set the learning rate to $3\times10^{-4}$ and the weight decay to $1\times10^{-4}$ for all experiments. Cosine annealing is used for learning-rate scheduling.

\textbf{U-Net~\citep{ronneberger2015unet} (under GPL-3.0 license)\footnote{https://github.com/milesial/Pytorch-UNet}.}
The hyperparameters used for U-Net are specified in Table~\ref{tab:baseline-params}. The model is trained for $250$ epochs using the AdamW optimizer with momentum parameters $(0.9,0.999)$. We set the learning rate to $8\times10^{-4}$ and the weight decay to $5\times10^{-3}$ for all experiments. Cosine annealing is used for learning-rate scheduling.

\textbf{Vision Transformer (ViT)~\citep{vit}.}
The hyperparameters used for ViT are configured as outlined in Table~\ref{tab:baseline-params}. The model is trained for $250$ epochs using the AdamW optimizer with momentum parameters $(0.9,0.999)$. We set the learning rate to $5\times10^{-4}$ and the weight decay to $1\times10^{-3}$ for all experiments. Cosine annealing is used for learning-rate scheduling.

\textbf{AtmoRep~\citep{atmorep} (under MIT license)\footnote{https://github.com/clessig/atmorep}.}
The hyperparameters used for AtmoRep are described in Table~\ref{tab:baseline-params}. We load pretrained weights from model ID \texttt{4nvwbetz}. The model is trained for $250$ epochs using the AdamW optimizer with momentum parameters $(0.9,0.999)$. We set the learning rate to $5\times10^{-4}$ and the weight decay to $2\times10^{-3}$ for all experiments. Cosine annealing is used for learning-rate scheduling.

\textbf{Prithvi~\citep{prithvi} (under MIT license)\footnote{https://github.com/NASA-IMPACT/Prithvi-WxC}.}
The hyperparameters used for Prithvi are described in Table~\ref{tab:baseline-params}. We load the official pretrained weights\footnote{https://huggingface.co/Prithvi-WxC/prithvi.wxc.2300m.v1} and finetune the input adapter, output head, and last decoder block. The model is trained for $100$ epochs using the AdamW optimizer with momentum parameters $(0.9,0.999)$. We set the learning rate to $3\times10^{-4}$ and the weight decay to $5\times10^{-4}$ for all experiments. Cosine annealing is used for learning-rate scheduling.

\subsection{Model efficiency}
\label{app:efficiency}

We further evaluate the computational efficiency of different models under the same hardware setting. All models are trained and evaluated on eight NVIDIA RTX 6000D GPUs with a per-GPU batch size of 1, corresponding to a total batch size of 8. Table~\ref{tab:efficiency} reports the number of parameters, training time per epoch, peak training VRAM, inference time on the test set, peak inference VRAM, and the corresponding reconstruction performance. REACT requires more computation than lightweight baselines because it jointly performs transport modeling, source correction, and carbonate-system regulation. Nevertheless, its cost remains substantially lower than large foundation-model baselines such as Prithvi, while achieving the best pH NRMSE and Chem. This indicates that REACT provides a favorable balance between reconstruction quality, chemical consistency, and computational cost.

\begin{table}[t]
\scriptsize
\centering
\caption{Parameter, efficiency, and performance comparison of different models on MOM6 monthly-to-weekly task under the 0.11\% sampling ratio. pH NRMSE is reported in units of $10^{-4}$, and Chem. is reported in units of $10^{-2}$.}
\label{tab:efficiency}

\setlength{\tabcolsep}{11pt}
\renewcommand{\arraystretch}{1}
\begin{tabular}{c c c c c c c c}
\toprule
\multirow{2}{*}{Methods} &
\multirow{2}{*}{Params} &
$\text{Time}_{\text{Train}}$ &
$\text{VRAM}_{\text{Train}}$ &
$\text{Time}_{\text{Infer}}$ &
$\text{VRAM}_{\text{Infer}}$ &
\multirow{2}{*}{pH NRMSE} &
\multirow{2}{*}{Chem.} \\
& &
(s/epoch) & (GB) &
(s) & (GB) &
& \\
\midrule

4D-Var. & 0.27M & 4.654 & 5.444 & 0.513 & 0.821 & 6.92\subpm{0.03} & 2.09\subpm{0.03} \\
Marble  & 0.83M & 1.237 & 0.563 & 0.213 & 0.274 & 25.7\subpm{0.03} & 5.65\subpm{0.10} \\
FNO     & 8.42M & 2.711 & 0.757 & 0.254 & 0.380 & 16.7\subpm{0.07} & 3.88\subpm{0.07} \\
U-Net   & 31.1M & 1.544 & 1.173 & 0.190 & 0.612 & 7.78\subpm{0.04} & 2.17\subpm{0.08} \\
ViT     & 0.11B & 8.805 & 2.319 & 0.298 & 1.093 & 9.56\subpm{0.10} & 2.51\subpm{0.01} \\
NUTS    & 0.15B & 17.865 & 3.640 & 0.361 & 1.722 & 6.74\subpm{0.06} & 1.95\subpm{0.10} \\
AtmoR.  & 0.80B & 73.072 & 18.429 & 1.267 & 6.399 & 14.6\subpm{0.16} & 3.52\subpm{0.04} \\
Prithvi & 2.37B & 1127.814 & 71.069 & 137.471 & 9.664 & 16.0\subpm{0.52} & 3.48\subpm{0.07} \\

\midrule
\textbf{REACT} & 0.18B & 100.745 & 32.087 & 84.035 & 6.118 & 5.84\subpm{0.04} & 1.59\subpm{0.41} \\
\bottomrule
\end{tabular}%
\end{table}

\section{Additional Experiments Results}
\label{app:additional_exps}

\subsection{Decomposed Chemical Errors for Table~\ref{tab:main_results}}
\label{app:main_results_decomposed_chemical}

Table~\ref{tab:main_results_pco2_dic} provides the \(p\mathrm{CO}_2^s\) and DIC components of the Chem. metric reported in Table~\ref{tab:main_results}.
Across both monthly-to-weekly and weekly-to-daily reconstruction tasks, REACT achieves the lowest decomposed chemical errors under all sampling ratios.
This confirms that the Chem. improvements in the main reconstruction results come from simultaneous reductions in both carbonate-system components, rather than from a single dominant term.

\begin{table}[t]
\scriptsize
\centering
\caption{Decomposed chemical errors for the main reconstruction results in Table~\ref{tab:main_results}. The table reports the two components of the chemical consistency metric, i.e., \(p\mathrm{CO}_2^s\) and DIC errors, under different sampling ratios for monthly-to-weekly and weekly-to-daily sea surface pH reconstruction. Both errors are reported in units of $10^{-2}$. All values are rounded to three significant digits. \textit{Params} denotes model size. \textit{U-Net$^{2}$} denotes a two-level U-Net variant.}
\label{tab:main_results_pco2_dic}

\setlength{\tabcolsep}{5.5pt}
\renewcommand{\arraystretch}{1}
\begin{tabular}{c c cc cc cc cc cc cc}
\toprule
\multirow{3}{*}{\textbf{Methods}} &
\multirow{3}{*}{\textbf{Params}} &
\multicolumn{6}{c}{\textbf{Month\(\to\)Week}} & \multicolumn{6}{c}{\textbf{Week\(\to\)Day}} \\
\cmidrule(lr){3-8}\cmidrule(lr){9-14}
& &
\multicolumn{2}{c}{0.11\%} & \multicolumn{2}{c}{1\%} & \multicolumn{2}{c}{10\%} &
\multicolumn{2}{c}{0.11\%} & \multicolumn{2}{c}{1\%} & \multicolumn{2}{c}{10\%} \\
\cmidrule(lr){3-4}\cmidrule(lr){5-6}\cmidrule(lr){7-8}
\cmidrule(lr){9-10}\cmidrule(lr){11-12}\cmidrule(lr){13-14}
& &
\(p\mathrm{CO}_2^s\) & DIC &
\(p\mathrm{CO}_2^s\) & DIC &
\(p\mathrm{CO}_2^s\) & DIC &
\(p\mathrm{CO}_2^s\) & DIC &
\(p\mathrm{CO}_2^s\) & DIC &
\(p\mathrm{CO}_2^s\) & DIC \\
\midrule

Kriging (Exp.) & -- & 10.3 & 9.82 & 7.33 & 6.78 & 4.86 & 4.26 & 10.2 & 9.73 & 7.16 & 6.72 & 4.45 & 3.99 \\
Kriging (Sph.) & -- & 10.3 & 9.99 & 8.00 & 7.65 & 6.41 & 5.82 & 10.4 & 9.91 & 7.52 & 7.18 & 5.30 & 4.77 \\
4D-VarNet      & 0.27M & 2.28 & 1.91 & 2.18 & 1.82 & 2.06 & 1.74 & 1.93 & 1.79 & 1.95 & 1.81 & 2.12 & 1.98 \\
Marble         & 0.83M & 5.95 & 5.36 & 5.99 & 5.42 & 5.77 & 5.25 & 5.66 & 5.35 & 5.67 & 5.44 & 5.64 & 5.39 \\
FNO            & 8.42M & 4.07 & 3.69 & 4.00 & 3.63 & 4.01 & 3.63 & 2.14 & 2.00 & 2.37 & 2.21 & 2.31 & 2.16 \\
U-Net          & 31.1M & 2.32 & 2.02 & 2.27 & 2.01 & 2.24 & 1.99 & 1.93 & 1.79 & 1.91 & 1.78 & 1.90 & 1.78 \\
U-Net$^{2}$    & 1.96M & 2.35 & 2.07 & 2.37 & 2.08 & 2.22 & 2.00 & 2.10 & 1.95 & 2.14 & 1.98 & 2.00 & 1.86 \\
ViT            & 0.11B & 2.65 & 2.36 & 2.68 & 2.41 & 2.64 & 2.35 & 1.76 & 1.57 & 2.23 & 2.12 & 1.94 & 1.82 \\
NUTS           & 0.15B & 2.07 & 1.82 & 2.19 & 1.92 & 2.42 & 2.18 & 1.75 & 1.63 & 1.71 & 1.60 & 1.83 & 1.72 \\
AtmoRep        & 0.80B & 3.74 & 3.30 & 4.24 & 3.67 & 3.08 & 2.65 & 2.44 & 2.26 & 2.16 & 1.99 & 2.14 & 1.96 \\
Prithvi        & 2.37B & 3.83 & 3.13 & 3.76 & 3.14 & 3.45 & 2.88 & 2.92 & 2.58 & 2.91 & 2.58 & 2.91 & 2.54 \\

\midrule
\oursrow
\textbf{REACT} & 0.18B & \textbf{1.73} & \textbf{1.45} & \textbf{1.65} & \textbf{1.40} & \textbf{1.67} & \textbf{1.41} & \textbf{1.44} & \textbf{1.27} & \textbf{1.42} & \textbf{1.28} & \textbf{1.40} & \textbf{1.25} \\
\bottomrule
\end{tabular}%
\end{table}

\subsection{Additional Qualitative Results}
\label{app:additional_qualitative_results}

\begin{figure}[t]
    \centering
    \captionsetup{font=footnotesize,skip=2pt}
    \includegraphics[width=0.95\linewidth]{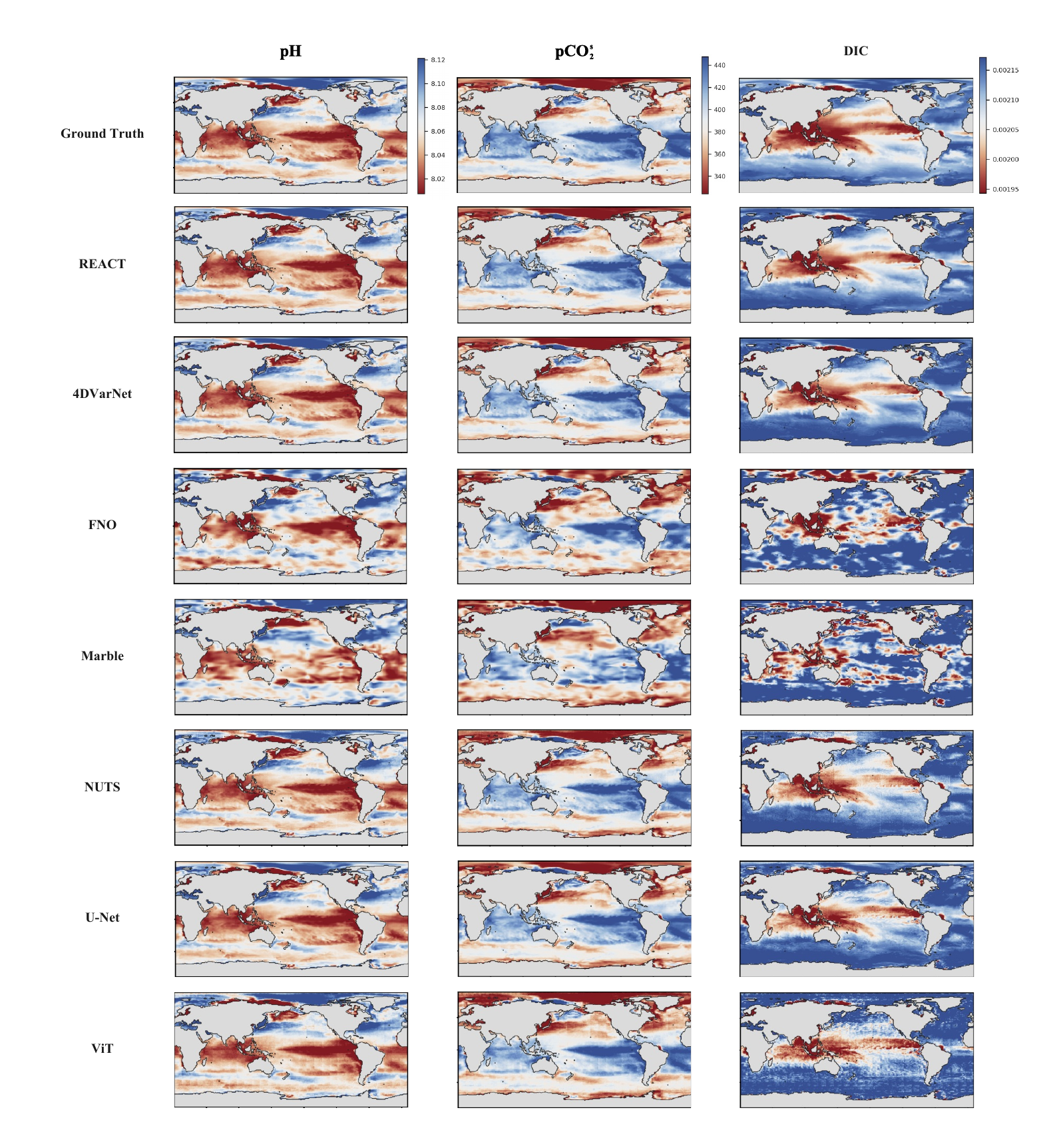}
    \caption{\textbf{Qualitative field comparison.} Ground-truth and reconstructed sea surface pH, \(p\mathrm{CO}_2^s\), and DIC fields from different methods at one randomly selected MOM6 test time slice.}
    \label{fig:qualitative_fields}
\end{figure}

\begin{figure}[t]
    \centering
    \captionsetup{font=footnotesize,skip=2pt}
    \includegraphics[width=0.95\linewidth]{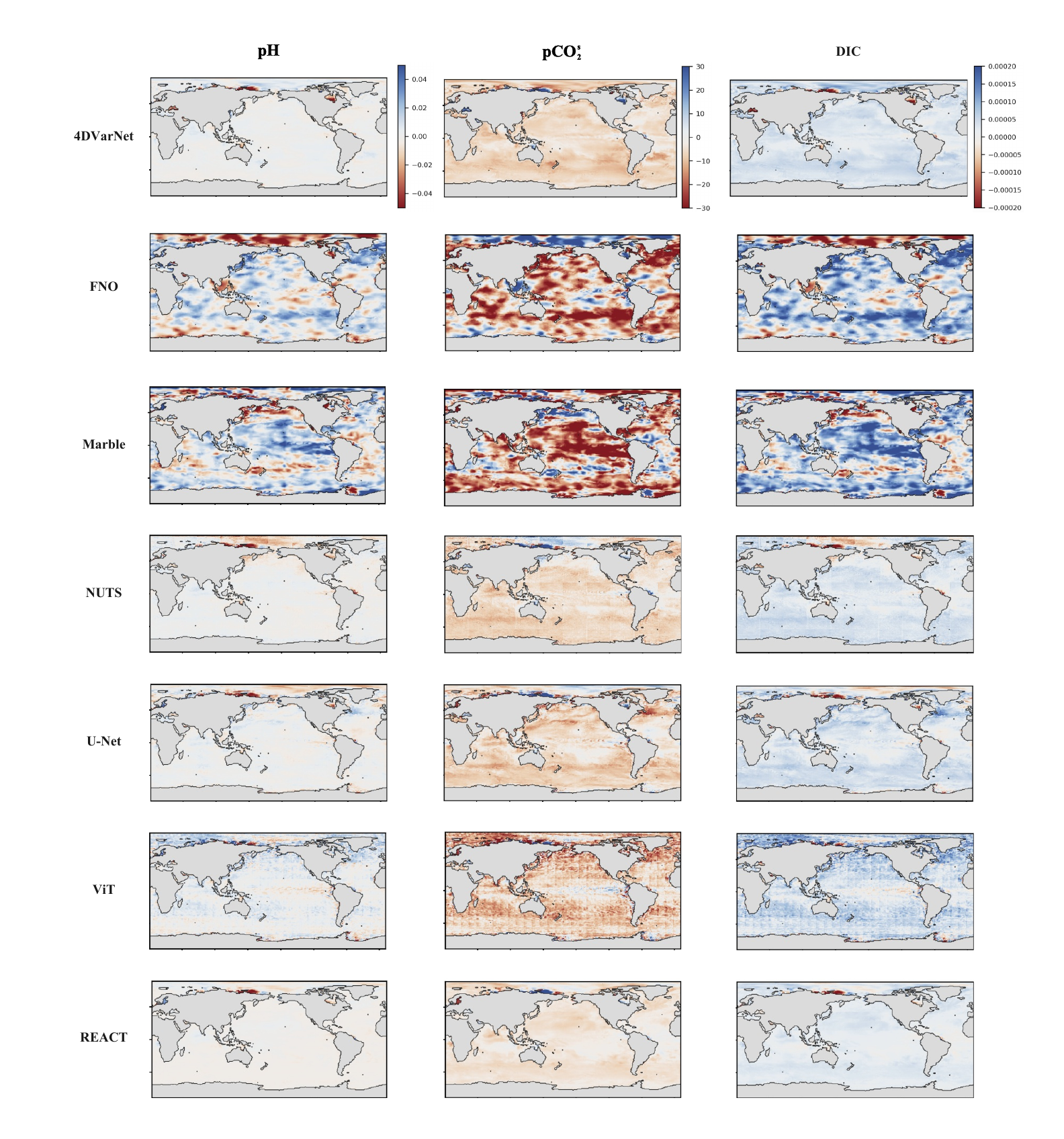}
    \caption{\textbf{Qualitative error maps.} Reconstruction errors of different methods for sea surface pH, \(p\mathrm{CO}_2^s\), and DIC at the same time slice, computed as prediction minus ground truth.}
    \label{fig:qualitative_errors}
\end{figure}

This subsection provides the additional qualitative results referenced in Figure~\ref{fig:qualitative}. Figure~\ref{fig:qualitative_fields} compares the ground-truth and reconstructed sea surface pH, \(p\mathrm{CO}_2^s\), and DIC fields from different methods at one randomly selected MOM6 test time slice. Figure~\ref{fig:qualitative_errors} shows the corresponding reconstruction errors, computed as prediction minus ground truth. The results show that REACT better preserves the spatial structures of the three variables and produces smaller residual errors than the baselines.

\begin{wrapfigure}[10]{r}{0.42\linewidth}
    \centering
    \vspace{-1em}
    \captionsetup{font=footnotesize,skip=2pt}
    \includegraphics[width=0.98\linewidth]{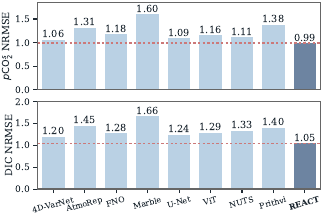}
    \caption{\textbf{Decomposed real-observation results.} NRMSE of \(p\mathrm{CO}_2^s\) and DIC for different methods on real WOD observations. Both metrics are reported in units of \(10^{-1}\).}
    \label{fig:wod_pco2_dic_results}
    \vspace{-0.8em}
\end{wrapfigure}

\subsection{Additional Real-Observation Results}
\label{app:wod_decomposed_results}

This subsection provides the decomposed real-observation results corresponding to Figure~\ref{fig:wod_real_results}. Figure~\ref{fig:wod_pco2_dic_results} reports the NRMSE of \(p\mathrm{CO}_2^s\) and DIC on the WOD real-observation experiment. Both metrics are reported in units of \(10^{-1}\). REACT achieves the lowest error on both carbonate variables, showing that its improvement on real observations is consistent not only for pH and Chem., but also for the two decomposed carbonate-system quantities.

\subsection{Decomposed Chemical Errors for Table~\ref{tab:module_ablation}}
\label{app:loss_ablation_decomposed_chemical}

Table~\ref{tab:loss_ablation_pco2_dic} provides the decomposed \(p\mathrm{CO}_2^s\) and DIC errors for the source-module and constraint ablations reported in Table~\ref{tab:module_ablation}. The results show that the chemical constraint consistently reduces both carbonate-system components. For REACT on the monthly-to-weekly reconstruction task, adding \(\mathcal L_{\mathrm{chem}}\) reduces the \(p\mathrm{CO}_2^s\) error from \(2.06\times10^{-2}\) to \(1.73\times10^{-2}\) and the DIC error from \(1.82\times10^{-2}\) to \(1.45\times10^{-2}\). The same pattern is observed on the weekly-to-daily reconstruction task and in the U-Net ablation, confirming that the chemical constraint improves carbonate closure beyond the REACT architecture alone.

\begin{table*}[t]
\centering
\scriptsize
\setlength{\tabcolsep}{11pt}
\renewcommand{\arraystretch}{1}
\caption{Decomposed chemical errors for the source module and constraint ablations. We report \(p\mathrm{CO}_2^s\) and DIC errors for the ablation variants in Table~\ref{tab:module_ablation}, covering REACT on Month$\to$Week and Week$\to$Day reconstruction under the 0.11\% sampling ratio, and the U-Net baseline ablation on Month$\to$Week reconstruction under the 0.11\% sampling ratio. Both metrics are reported in units of $10^{-2}$. All values are rounded to three significant digits.}
\label{tab:loss_ablation_pco2_dic}
\begin{tabular}{c cc cc cc cc}
\toprule
&
\multicolumn{2}{c}{} &
\multicolumn{2}{c}{REACT (Month$\to$Week)} &
\multicolumn{2}{c}{REACT (Week$\to$Day)} &
\multicolumn{2}{c}{U-Net (Month$\to$Week)} \\
\cmidrule(lr){4-5}\cmidrule(lr){6-7}\cmidrule(lr){8-9}
Source &
$\mathcal{L}_{\mathrm{mass}}$ &
$\mathcal{L}_{\mathrm{chem}}$ &
\(p\mathrm{CO}_2^s\) &
DIC &
\(p\mathrm{CO}_2^s\) &
DIC &
\(p\mathrm{CO}_2^s\) &
DIC \\
\midrule
--         & --         & --         & 8.94 & Div. & 2.84 & 2.62 & 2.27 & 1.98 \\
\checkmark & --         & --         & 2.06 & 1.80 & 1.71 & 1.59 & 2.28 & 2.00 \\
\checkmark & \checkmark & --         & 2.06 & 1.82 & 1.76 & 1.62 & 2.17 & 1.89 \\
\checkmark & \checkmark & \checkmark & 1.73 & 1.45 & 1.44 & 1.27 & 1.77 & 1.53 \\
\bottomrule
\end{tabular}
\end{table*}

\subsection{Decomposed Chemical Errors for Monthly-to-Weekly Design Ablations}
\label{app:monthly_weekly_ablations_decomposed_chemical}

Table~\ref{tab:monthly_weekly_ablations_pco2_dic} provides the decomposed \(p\mathrm{CO}_2^s\) and DIC errors for the monthly-to-weekly design ablations reported in Table~\ref{tab:ablations_choices}. The decomposed results show that most architectural choices have similar carbonate-system errors around the default setting, while removing key auxiliary variables causes larger degradation. In particular, removing sea surface \(p\mathrm{CO}_2^s\) substantially increases both decomposed errors, indicating that sparse \(p\mathrm{CO}_2^s\) observations are important for maintaining carbonate consistency. The 3D CNN initializer also becomes unstable, consistent with the divergent behavior reported in the main ablation table.

\begin{table*}[t]
\centering
\scriptsize
\setlength{\tabcolsep}{4pt}
\renewcommand{\arraystretch}{1.03}
\caption{Decomposed chemical errors for the monthly-to-weekly ablation study in Table~\ref{tab:ablations_choices}. The table reports \(p\mathrm{CO}_2^s\) and DIC errors in units of $10^{-2}$ for interval length, state initializer depth, pH decoder/source-module depth, state initializer architecture, and auxiliary-variable ablations. The default setting uses interval length $4$, state initializer depth $12$, pH decoder/source-module depth $4$, a ViViT state initializer, all auxiliary variables, and both mass and chemistry constraints. All values are rounded to three significant digits.}
\label{tab:monthly_weekly_ablations_pco2_dic}

\begin{minipage}[t]{0.475\textwidth}
\centering
{\renewcommand{\arraystretch}{1.05}
\begin{tabular*}{\linewidth}{@{\extracolsep{\fill}}lccc@{}}
\toprule
Factor & Setting & \(p\mathrm{CO}_2^s\) & DIC \\
\midrule

\multirow{4}{*}{Interval length}
& 1 & 1.71 & 1.44 \\
& 2 & 1.66 & 1.36 \\
& 4 & 1.73 & 1.45 \\
& 8 & 1.70 & 1.46 \\

\midrule
\multirow{4}{*}{Initializer depth}
& 4  & 1.73 & 1.46 \\
& 6  & 1.74 & 1.46 \\
& 12 & 1.73 & 1.45 \\
& 24 & 1.71 & 1.47 \\

\midrule
\multirow{4}{*}{Decoder/source depth}
& 2 & 1.66 & 1.42 \\
& 3 & 1.67 & 1.42 \\
& 4 & 1.73 & 1.45 \\
& 5 & 1.99 & 1.72 \\

\bottomrule
\end{tabular*}}
\end{minipage}
\hfill
\begin{minipage}[t]{0.475\textwidth}
\centering
\begin{tabular*}{\linewidth}{@{\extracolsep{\fill}}lccc@{}}
\toprule
Factor & Setting & \(p\mathrm{CO}_2^s\) & DIC \\
\midrule

\multirow{4}{*}{Initializer}
& ViT    & 1.66 & 1.40 \\
& 2D CNN & 1.80 & 1.53 \\
& 3D CNN & 7.83 & 916 \\
& ViViT  & 1.73 & 1.45 \\

\midrule
\multirow{9}{*}{Auxiliary removed}
& \(p\mathrm{CO}_2^s\) & 5.22 & 4.73 \\
& \(p\mathrm{CO}_2^a\) & 1.72 & 1.39 \\
& Chl.      & 1.84 & 1.63 \\
& SSH       & 1.72 & 1.44 \\
& MLD       & 1.72 & 1.42 \\
& Salt      & 1.82 & 1.59 \\
& Temp.     & 1.75 & 1.47 \\
& $u$       & 1.75 & 1.46 \\
& $v$       & 1.90 & 1.56 \\

\bottomrule
\end{tabular*}
\end{minipage}
\end{table*}

\subsection{Weekly-to-Daily Design Ablations}
\label{app:weekly_daily_design_ablations}

Table~\ref{tab:weekly_daily_ablation_full} provides the weekly-to-daily counterpart of the monthly-to-weekly design ablations in Table~\ref{tab:ablations_choices}. The default configuration still provides the best overall trade off. Although the model with an initializer depth of 24 achieves better overall performance than the default configuration, we retain a depth of 12 due to computational cost considerations. The auxiliary-variable ablations again show that removing \(p\mathrm{CO}_2^s\) causes the largest degradation in both decomposed carbonate errors, confirming its importance for chemically consistent reconstruction across temporal scales.

\begin{table*}[t]
\centering
\scriptsize
\setlength{\tabcolsep}{3.5pt}
\renewcommand{\arraystretch}{1.03}
\caption{Weekly-to-daily counterpart of the ablation study in Table~\ref{tab:ablations_choices} under the $0.11\%$ sampling ratio. The table reports full metrics for interval length, state initializer depth, pH decoder/source-module depth, and state initializer architecture, with auxiliary-variable ablations shown separately using decomposed chemical errors. The default setting uses interval length $2$, state initializer depth $12$, pH decoder/source-module depth $4$, a ViViT state initializer, all auxiliary variables, and both mass and chemistry constraints. pH NRMSE is reported in units of $10^{-4}$, while \(p\mathrm{CO}_2^s\), DIC, and Chem. are reported in units of $10^{-2}$. Drift values are reported in per mille. All values are rounded to three significant digits. \textit{Div.} denotes divergent results.}
\label{tab:weekly_daily_ablation_full}

\begin{tabular*}{\textwidth}{@{\extracolsep{\fill}}lccccccc@{}}
\toprule
Factor & Setting & pH & \(p\mathrm{CO}_2^s\) & DIC & Chem. & Drift$_{\mathrm{local}}$ & Drift$_{\mathrm{global}}$ \\
\midrule

\multirow{4}{*}{Interval length}
& 1 & 4.89 & 1.47 & 1.29 & 1.38 & 0.61\textperthousand & Div. \\
& 2 & 4.64 & 1.44 & 1.27 & 1.35 & 0.63\textperthousand & 7.68\textperthousand \\
& 3 & 5.37 & 1.59 & 1.50 & 1.54 & 0.85\textperthousand & 8.35\textperthousand \\
& 4 & 4.79 & 1.35 & 1.19 & 1.27 & 0.77\textperthousand & 3.59\textperthousand \\

\midrule
\multirow{4}{*}{Initializer depth}
& 4  & 4.74 & 1.35 & 1.20 & 1.27 & 0.63\textperthousand & 8.21\textperthousand \\
& 6  & 4.80 & 1.33 & 1.18 & 1.26 & 0.60\textperthousand & 7.96\textperthousand \\
& 12 & 4.64 & 1.44 & 1.27 & 1.35 & 0.63\textperthousand & 7.68\textperthousand \\
& 24 & 4.56 & 1.40 & 1.26 & 1.33 & 0.60\textperthousand & 7.79\textperthousand \\

\midrule
\multirow{4}{*}{Decoder/source depth}
& 2 & 4.58 & 1.44 & 1.30 & 1.37 & 0.75\textperthousand & Div. \\
& 3 & 4.89 & 1.33 & 1.16 & 1.25 & 0.64\textperthousand & Div. \\
& 4 & 4.64 & 1.44 & 1.27 & 1.35 & 0.63\textperthousand & 7.68\textperthousand \\
& 5 & 4.91 & 1.25 & 1.10 & 1.17 & 0.53\textperthousand & 12.1\textperthousand \\

\midrule
\multirow{4}{*}{Initializer}
& ViT    & 4.91 & 1.41 & 1.25 & 1.33 & 0.51\textperthousand & 7.01\textperthousand \\
& 2D CNN & 4.76 & 1.33 & 1.18 & 1.26 & 0.54\textperthousand & Div. \\
& 3D CNN & 4.78 & 1.42 & 1.26 & 1.34 & 0.57\textperthousand & Div. \\
& ViViT  & 4.64 & 1.44 & 1.27 & 1.35 & 0.63\textperthousand & 7.68\textperthousand \\

\bottomrule
\end{tabular*}

\vspace{6pt}

\begin{tabular*}{\textwidth}{@{\extracolsep{\fill}}cccccccccc@{}}
\toprule
Metric & \(p\mathrm{CO}_2^s\) & \(p\mathrm{CO}_2^a\) & Chl. & SSH & MLD & Salt & Temp. & $u$ & $v$ \\
\midrule
\(p\mathrm{CO}_2^s\) & 2.70 & 1.40 & 1.38 & 1.34 & 1.33 & 1.47 & 1.37 & 1.33 & 1.39 \\
DIC      & 2.44 & 1.27 & 1.22 & 1.19 & 1.18 & 1.32 & 1.21 & 1.19 & 1.23 \\
\bottomrule
\end{tabular*}

\end{table*}

\subsection{Uncertainty Estimation}
\label{app:uncertainty_results}

To evaluate predictive uncertainty, we start from the converged deterministic REACT checkpoint and freeze the reconstruction model.
We then train an uncertainty head to estimate the standard deviation of a Gaussian predictive distribution around the deterministic prediction.
The predictive mean is fixed by the deterministic model, while the uncertainty head learns the error scale around this mean.
For the selected checkpoint, the one-standard-deviation coverage values are $0.821$, $0.752$, and $0.640$ for pH, \(p\mathrm{CO}_2^s\), and DIC, respectively.
The corresponding two-standard-deviation coverage values are $0.963$, $0.971$, and $0.958$.
These two-standard-deviation coverages are close to the nominal Gaussian coverage of $0.95$, suggesting that the learned uncertainty intervals are broadly consistent with the observed reconstruction errors.

\begin{figure}[t]
    \centering
    \captionsetup{font=footnotesize,skip=2pt}
    \includegraphics[width=0.48\linewidth]{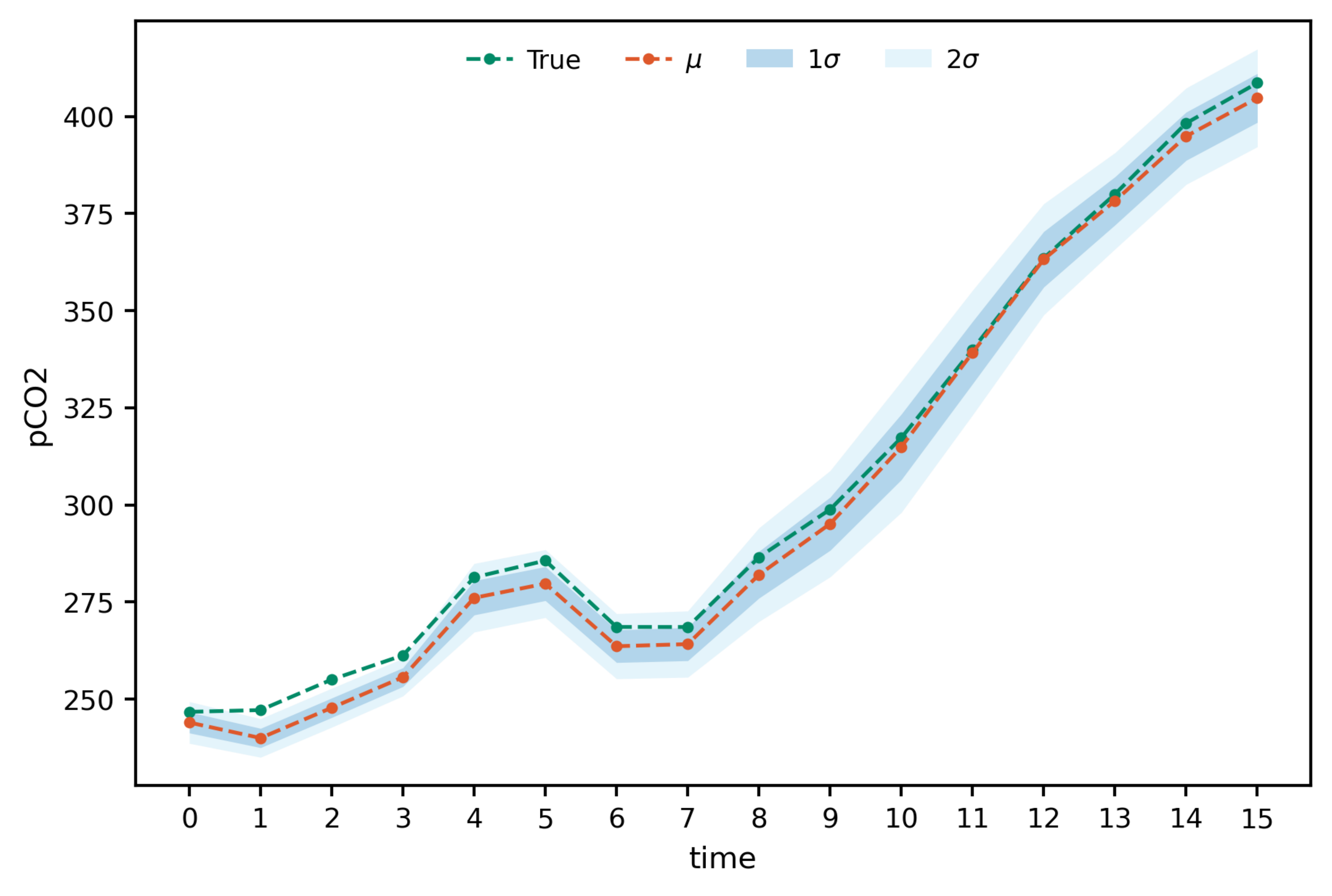}
    \hfill
    \includegraphics[width=0.48\linewidth]{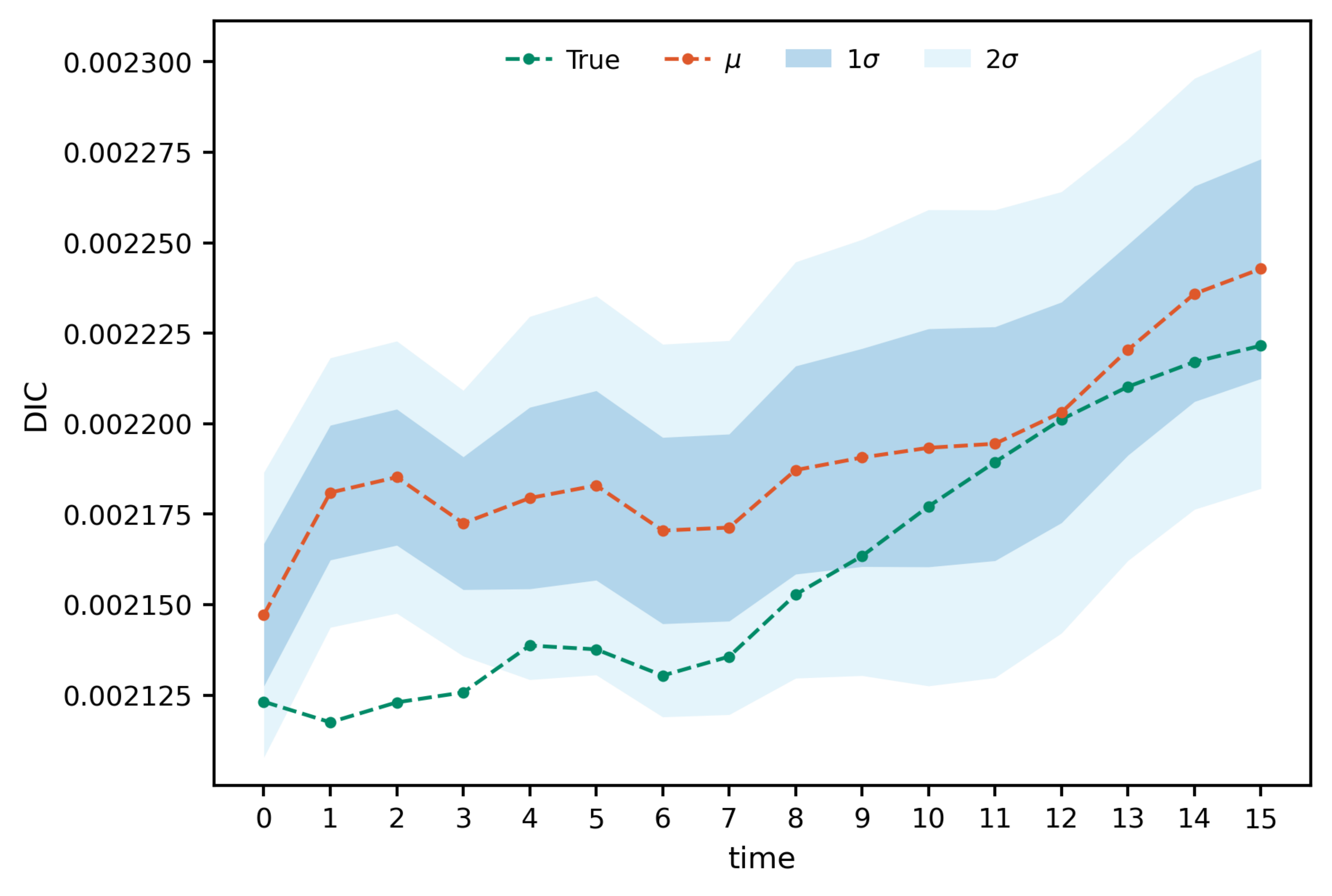}
    \caption{\textbf{Uncertainty estimation of REACT for \(p\mathrm{CO}_2^s\) and DIC.}
    The predictive means are shown with one- and two-standard-deviation intervals.}
    \label{figure:uncertainty_pco2_dic}
\end{figure}

\section{Limitations}
\label{app:limitations}
A limitation of the REACT model is that the learned source module lacks full interpretability. Although the source term improves reconstruction by capturing non-conservative variations in the carbon cycle, the learned correction is not explicitly decomposed into individual physical or biogeochemical mechanisms, such as air--sea exchange, biological activity, or unresolved mixing. Consequently, while the source module provides a flexible way to account for processes that are not directly resolved by the transport equation, its attribution to underlying mechanisms remains limited.

Furthermore, REACT is currently designed for sea surface pH reconstruction and does not explicitly account for vertical carbon transport or subsurface carbonate dynamics. This limits its ability to characterize the three-dimensional structure of ocean carbonate chemistry, which is important for understanding carbon-cycle variability and marine biogeochemical processes. Extending REACT to three-dimensional ocean fields is therefore a promising direction for future work.

\newpage

\section*{NeurIPS Paper Checklist}

\begin{enumerate}

\item {\bf Claims}
    \item[] Question: Do the main claims made in the abstract and introduction accurately reflect the paper's contributions and scope?
    \item[] Answer: \answerYes{} 
    \item[] Justification: The main claims and scope are clearly stated in the abstract and Introduction~\ref{sec:introduction}.

\item {\bf Limitations}
    \item[] Question: Does the paper discuss the limitations of the work performed by the authors?
    \item[] Answer: \answerYes{} 
    \item[] Justification: We discuss the limitations of REACT in Appendix~\ref{app:limitations}.

\item {\bf Theory assumptions and proofs}
    \item[] Question: For each theoretical result, does the paper provide the full set of assumptions and a complete (and correct) proof?
    \item[] Answer: \answerYes{} 
    \item[] Justification: The full assumptions, formal statements, and proofs are provided in Appendix~\ref{app:react-theory}.

    \item {\bf Experimental result reproducibility}
    \item[] Question: Does the paper fully disclose all the information needed to reproduce the main experimental results of the paper to the extent that it affects the main claims and/or conclusions of the paper (regardless of whether the code and data are provided or not)?
    \item[] Answer: \answerYes{} 
    \item[] Justification: We provide detailed experimental settings in Appendix~\ref{app:experiment_details} and hyperparameter configurations in Appendix~\ref{app:model_params}; we also provide an anonymous repository containing our code and data through \href{https://anonymous.4open.science/r/REACT-67E7}{this link}.

\item {\bf Open access to data and code}
    \item[] Question: Does the paper provide open access to the data and code, with sufficient instructions to faithfully reproduce the main experimental results, as described in supplemental material?
    \item[] Answer: \answerYes{} 
    \item[] Justification: We provide the code and data through an anonymous repository available at \href{https://anonymous.4open.science/r/REACT-67E7}{this link}.

\item {\bf Experimental setting/details}
    \item[] Question: Does the paper specify all the training and test details (e.g., data splits, hyperparameters, how they were chosen, type of optimizer) necessary to understand the results?
    \item[] Answer: \answerYes{} 
    \item[] Justification: We provide the data split protocol in Table~\ref{tab:data_split} and report all training details in Appendix~\ref{app:experiment_details}.

\item {\bf Experiment statistical significance}
    \item[] Question: Does the paper report error bars suitably and correctly defined or other appropriate information about the statistical significance of the experiments?
    \item[] Answer: \answerYes{} 
    \item[] Justification: We report the standard errors in the main results table, as shown in Table~\ref{tab:main_results}.

\item {\bf Experiments compute resources}
    \item[] Question: For each experiment, does the paper provide sufficient information on the computer resources (type of compute workers, memory, time of execution) needed to reproduce the experiments?
    \item[] Answer: \answerYes{} 
    \item[] Justification: We provide the hardware setting and computational cost details in Appendix~\ref{app:efficiency}, including the GPU type, training time, inference time, and peak VRAM usage.
    
\item {\bf Code of ethics}
    \item[] Question: Does the research conducted in the paper conform, in every respect, with the NeurIPS Code of Ethics \url{https://neurips.cc/public/EthicsGuidelines}?
    \item[] Answer: \answerYes{} 
    \item[] Justification: This study follows the ethical guidelines and requirements specified by NeurIPS.

\item {\bf Broader impacts}
    \item[] Question: Does the paper discuss both potential positive societal impacts and negative societal impacts of the work performed?
    \item[] Answer: \answerYes{} 
    \item[] Justification: We discuss broader impacts in Conclusion~\ref{sec:conclusion}, including the potential extension of REACT to other marine active tracers.
    
\item {\bf Safeguards}
    \item[] Question: Does the paper describe safeguards that have been put in place for responsible release of data or models that have a high risk for misuse (e.g., pre-trained language models, image generators, or scraped datasets)?
    \item[] Answer: \answerNA{} 
    \item[] Justification: This work does not involve data, models, or released assets with high misuse or dual-use risks.

\item {\bf Licenses for existing assets}
    \item[] Question: Are the creators or original owners of assets (e.g., code, data, models), used in the paper, properly credited and are the license and terms of use explicitly mentioned and properly respected?
    \item[] Answer: \answerNo{} 
    \item[] Justification: All code implementations are cited with license details in Appendix~\ref{app:model_params}. For the WOD datasets, although we were unable to identify their specific licenses, we cite their official sources and provide the official links.

\item {\bf New assets}
    \item[] Question: Are new assets introduced in the paper well documented and is the documentation provided alongside the assets?
    \item[] Answer: \answerYes{} 
    \item[] Justification: We provide the newly released code and data, together with documentation for reproducing the experiments, through an anonymized repository available at \href{https://anonymous.4open.science/r/REACT-67E7}{this link}.

\item {\bf Crowdsourcing and research with human subjects}
    \item[] Question: For crowdsourcing experiments and research with human subjects, does the paper include the full text of instructions given to participants and screenshots, if applicable, as well as details about compensation (if any)? 
    \item[] Answer: \answerNA{} 
    \item[] Justification: This work does not involve crowdsourcing experiments or research with human subjects.

\item {\bf Institutional review board (IRB) approvals or equivalent for research with human subjects}
    \item[] Question: Does the paper describe potential risks incurred by study participants, whether such risks were disclosed to the subjects, and whether Institutional Review Board (IRB) approvals (or an equivalent approval/review based on the requirements of your country or institution) were obtained?
    \item[] Answer: \answerNA{} 
    \item[] Justification: This work does not involve crowdsourcing experiments or research with human subjects.

\item {\bf Declaration of LLM usage}
    \item[] Question: Does the paper describe the usage of LLMs if it is an important, original, or non-standard component of the core methods in this research? Note that if the LLM is used only for writing, editing, or formatting purposes and does \emph{not} impact the core methodology, scientific rigor, or originality of the research, declaration is not required.
    \item[] Answer: \answerNA{} 
    \item[] Justification: The core method development in this work does not involve LLMs as important, original, or non-standard components.

\end{enumerate}

\end{document}